\documentclass{article}

\usepackage[T1]{fontenc}

\usepackage{microtype}
\usepackage{graphicx}
\usepackage{subcaption}
\usepackage{booktabs} 
\usepackage{tabularx}
\usepackage{multirow}
\usepackage{array}

\usepackage{hyperref}

\usepackage[preprint]{icml2026}

\usepackage{amsmath}
\usepackage{amssymb}
\usepackage{mathtools}
\usepackage{amsthm}

\usepackage[capitalize,noabbrev]{cleveref}

\usepackage{tcolorbox}

\theoremstyle{plain}
\newtheorem{theorem}{Theorem}[section]

\theoremstyle{definition}
\newtheorem{definition}[theorem]{Definition}

\theoremstyle{remark}

\usepackage[textsize=tiny]{todonotes}

\usepackage{algorithm}
\usepackage{algorithmic}
\usepackage{xcolor}
\usepackage{float} 
\usepackage{tikz}
\usetikzlibrary{positioning, calc, arrows.meta, shapes, decorations.pathreplacing}

\icmltitlerunning{Expert Threshold Routing}

\begin{document}

\twocolumn[
\icmltitle{Expert Threshold Routing for Autoregressive Language Modeling with Dynamic Computation Allocation and Load Balancing}




\begin{icmlauthorlist}
\icmlauthor{Ryan Sun}{lehigh}
\icmlauthor{Yixin Liu}{lehigh}
\icmlauthor{Yonghui Wu}{UF}
\icmlauthor{Lichao Sun}{lehigh}
\end{icmlauthorlist}

\icmlaffiliation{lehigh}{Computer Science and Engineering, Lehigh University, Bethlehem, PA, USA}
\icmlaffiliation{UF}{MD-HOBI-BIOMED INFORMATICS, University of Florida, Gainesville, FL, USA}
\icmlcorrespondingauthor{Lichao Sun}{lis221@lehigh.edu}

\icmlkeywords{Large Language Models, Mixture of Experts, Sparse Architectures, Expert Choice}

\vskip 0.3in
]



\printAffiliationsAndNotice{Code available at \href{https://github.com/MasterGodzilla/Expert-Threshold-Routing}{GitHub repository}.}  

\begin{abstract}

Token-choice Mixture-of-Experts (TC-MoE) routes each token to a fixed number of experts, limiting dynamic computation allocation and requiring auxiliary losses to maintain load balance. We propose Expert Threshold (ET) routing, where each expert maintains an exponential moving average (EMA) threshold estimated from the global token distribution. At both training and inference, each token is independently routed to an expert if its score exceeds the expert's threshold, enabling dynamic computation allocation while achieving load balance without auxiliary losses. This fully causal mechanism eliminates dependence on other tokens in the batch, making it well-suited for autoregressive language modeling. In pretraining experiments scaling to 2.4B parameters on FineWeb-Edu, ET achieves 0.067 lower cross-entropy loss than TC-MoE, equivalent to reaching the same performance with 1.6× fewer tokens.

\end{abstract}

\section{Introduction}

\begin{figure}[!t]
\centering
\includegraphics[width=\linewidth]{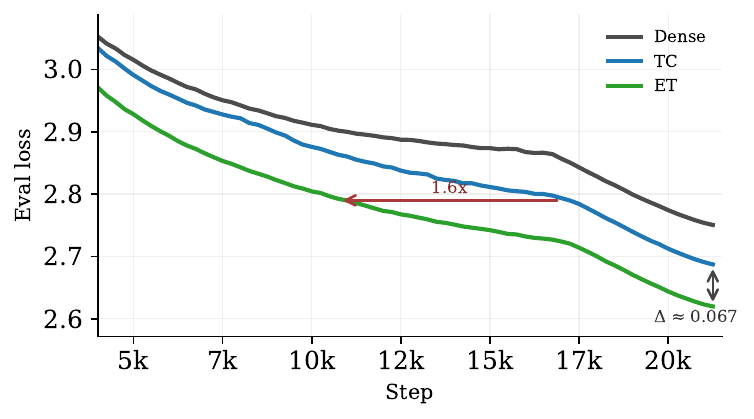}
\caption{evaluation loss for Dense, TC, and ET. Compared to TC, ET achieves a 0.067 final loss gap (TC vs ET), or equivalently reaching same performance level with 1.6x few tokens.}
\label{fig:d20_eval_loss_dense_tc_et}
\end{figure}


\begin{figure*}[t]
\centering
\resizebox{\textwidth}{!}{%
\begin{tikzpicture}[
    font=\sffamily,
    >={Stealth[length=4pt]},
]

\colorlet{tcblue}{blue!70!black}
\colorlet{tcfill}{blue!15}
\colorlet{ecorange}{orange!80!black}
\colorlet{ecfill}{orange!45}
\colorlet{gecteal}{teal!80!black}
\colorlet{gecfill}{teal!50}

\def\tcX{0}
\def\ecX{6.9}     
\def\seqX{4.7}
\def\batX{9.2}
\def\gecX{14.3}

\def\titleY{6.2}
\def\tagY{5.52}
\def\formulaY{5.0}

\node[font=\large\bfseries, tcblue] at (\tcX, \titleY) {Token Choice (TC)};
\node[font=\small\bfseries, red!70!black, anchor=base] at (\tcX, \tagY) {Load Imbalance};

\def\tokw{0.36}
\foreach \j [count=\idx from 0] in {1,...,7} {
    \pgfmathsetmacro{\tx}{\tcX - 3*\tokw + \idx*\tokw}
    \pgfmathtruncatemacro{\issel}{(\idx==3)}
    \ifnum\issel>0
        \node[draw, minimum width=\tokw cm - 0.04cm, minimum height=\tokw cm - 0.04cm,
              inner sep=0pt, fill=tcfill] (tok\j) at (\tx, 1.90) {};
    \else
        \node[draw=gray!50, minimum width=\tokw cm - 0.04cm, minimum height=\tokw cm - 0.04cm,
              inner sep=0pt, fill=gray!12, fill opacity=0.55, draw opacity=0.65] (tok\j) at (\tx, 1.90) {};
    \fi
}
\node[font=\scriptsize, tcblue] at (\tcX, 1.55) {token input};

\node[font=\small, tcblue] at (\tcX, 0.96) {$\mathrm{Top}_G\!\bigl(\{r_{t,i}\}_{\forall i}\bigr)$};

\foreach \i/\xoff in {1/-1.6, 2/-0.55, 3/0.55, 4/1.6} {
    \node[draw, rounded corners=4pt, minimum width=0.95cm, minimum height=1.0cm,
          fill=tcfill, align=center, font=\small]
          (exp\i) at (\tcX+\xoff, 3.65) {$\mathcal{M}_{\i}$};
}

\draw[->, tcblue!70, semithick] (tok4.north) -- (exp1.south);

\draw[->, tcblue!35, semithick, opacity=0.55] (tok2.north) -- (exp1.south);
\draw[->, tcblue!35, semithick, opacity=0.55] (tok6.north) -- (exp3.south);
\draw[->, tcblue!35, semithick, opacity=0.55] (tok1.north) -- (exp4.south);
\draw[->, tcblue!35, semithick, opacity=0.55] (tok3.north) -- (exp4.south);
\draw[->, tcblue!35, semithick, opacity=0.55] (tok5.north) -- (exp4.south);
\draw[->, tcblue!35, semithick, opacity=0.55] (tok7.north) -- (exp4.south);

\foreach \xoff/\n in {-1.6/2, -0.55/0, 0.55/1, 1.6/4} {
    \foreach \k in {1,...,4} {
        \ifnum\k>\n\relax
        \else
            \fill[tcblue!80] (\tcX+\xoff, 4.08 + 0.20*\k) circle (1.8pt);
        \fi
    }
}

\node[font=\large\bfseries, ecorange] at (\ecX, \titleY) {Expert Choice (EC)};
\node[font=\small\bfseries, red!70!black, anchor=base] at (\ecX, \tagY) {Non Causal};

\node[font=\small\bfseries, ecorange] at (\seqX, \formulaY) {Seq top-$k$};

\node[draw, rounded corners=4pt, minimum width=1.0cm, minimum height=0.8cm,
      fill=ecfill!50, align=center, font=\small]
      (seqexp) at (\seqX, 4.2) {$\mathcal{M}_i$};

\def\sw{0.42}
\foreach \j [count=\idx from 0] in {1,...,8} {
    \pgfmathsetmacro{\bx}{\seqX - 3.5*\sw + \idx*\sw}
    \pgfmathtruncatemacro{\sel}{(\idx==1) + (\idx==4) + (\idx==6)}
    \ifnum\sel>0
        \node[draw, minimum width=\sw cm - 0.04cm, minimum height=\sw cm - 0.04cm,
              inner sep=0pt, fill=ecfill] (st\j) at (\bx, 2.8) {};
    \else
        \node[draw, minimum width=\sw cm - 0.04cm, minimum height=\sw cm - 0.04cm,
              inner sep=0pt, fill=white] (st\j) at (\bx, 2.8) {};
    \fi
}

\draw[->, ecorange, semithick] (seqexp.south) -- (st2.north);
\draw[->, ecorange, semithick] (seqexp.south) -- (st5.north);
\draw[->, ecorange, semithick] (seqexp.south) -- (st7.north);

\draw[decorate, decoration={brace, amplitude=5pt, mirror}, thick, ecorange]
    (st1.south west) ++(0,-0.1) coordinate (seqBraceL)
    -- (st8.south east |- seqBraceL) ++(0,-0.1) coordinate (seqBraceR)
    node[midway, below=7pt, font=\scriptsize, ecorange] {1 sequence};

\node[font=\small, ecorange] at (\seqX, 0.96) {$\mathrm{Top}_k\!\bigl(\{r_{t,i}\}_{\forall t\in\mathrm{seq}}\bigr)$};

\node[font=\small\bfseries, ecorange] at (\batX, \formulaY) {Batch top-$k$};

\node[draw, rounded corners=4pt, minimum width=1.0cm, minimum height=0.8cm,
      fill=ecfill!50, align=center, font=\small]
      (batexp) at (\batX, 4.2) {$\mathcal{M}_i$};

\def\gw{0.42}
\def\gh{0.42}
\pgfmathsetmacro{\gridLeft}{\batX - 2.5*\gw}

\foreach \r in {0,1,2} {
    \foreach \c in {0,...,5} {
        \pgfmathsetmacro{\gx}{\gridLeft + \c*\gw}
        \pgfmathsetmacro{\gy}{3.2 - \r*\gh}
        \pgfmathtruncatemacro{\issel}{%
            (\r==0 && \c==1) + (\r==0 && \c==4) +
            (\r==1 && \c==3) +
            (\r==2 && \c==0) + (\r==2 && \c==5)}
        \ifnum\issel>0
            \node[draw, minimum width=\gw cm - 0.04cm, minimum height=\gh cm - 0.04cm,
                  inner sep=0pt, fill=ecfill] (g\r\c) at (\gx, \gy) {};
        \else
            \node[draw, minimum width=\gw cm - 0.04cm, minimum height=\gh cm - 0.04cm,
                  inner sep=0pt, fill=white] (g\r\c) at (\gx, \gy) {};
        \fi
    }
}

\node[font=\scriptsize, anchor=east] at (\gridLeft - 0.3, 3.2) {$s_1$};
\node[font=\scriptsize, anchor=east] at (\gridLeft - 0.3, 3.2-\gh) {$s_2$};
\node[font=\scriptsize, anchor=east] at (\gridLeft - 0.3, 3.2-2*\gh) {$s_3$};

\draw[->, ecorange, semithick] (batexp.south) -- (g01.north);
\draw[->, ecorange, semithick] (batexp.south) -- (g04.north);
\draw[->, ecorange, semithick] (batexp.south) -- (g13.north);
\draw[->, ecorange, semithick] (batexp.south) -- (g20.north);
\draw[->, ecorange, semithick] (batexp.south) -- (g25.north);

\draw[decorate, decoration={brace, amplitude=5pt, mirror}, thick, ecorange]
    (g20.south west) ++(0,-0.1) coordinate (batBraceL)
    -- (g25.south east |- batBraceL) ++(0,-0.1) coordinate (batBraceR)
    node[midway, below=7pt, font=\scriptsize, ecorange] {1 batch};

\node[font=\small, ecorange] at (\batX, 0.96) {$\mathrm{Top}_k\!\bigl(\{r_{t,i}\}_{\forall t\in\mathrm{batch}}\bigr)$};

\node[font=\large\bfseries, gecteal] at (\gecX, \titleY) {Expert Threshold (ET)};
\node[font=\small\bfseries, green!50!black, anchor=base] at (\gecX, \tagY) {Fully Causal};

\def\bw{0.44}
\def\bgap{0.12}
\def\bbase{2.0}
\def\bstart{12.3}
\pgfmathsetmacro{\bstep}{\bw+\bgap}

\def\bscale{0.56}
\def\thresh{2.60}

\foreach \b/\h/\above in {0/1.0/0, 1/3.6/1, 2/2.4/0, 3/0.7/0, 4/4.2/1, 5/1.8/0, 6/2.9/1, 7/0.5/0} {
    \pgfmathsetmacro{\bx}{\bstart + \b*\bstep}
    \pgfmathsetmacro{\btop}{\bbase + \h*\bscale}
    \ifnum\above=1
        \fill[gecfill] (\bx-\bw/2, \bbase) rectangle (\bx+\bw/2, \btop);
        \draw[gecteal] (\bx-\bw/2, \bbase) rectangle (\bx+\bw/2, \btop);
    \else
        \fill[white] (\bx-\bw/2, \bbase) rectangle (\bx+\bw/2, \btop);
        \draw[gray!55] (\bx-\bw/2, \bbase) rectangle (\bx+\bw/2, \btop);
    \fi
}

\pgfmathsetmacro{\thY}{\bbase + \thresh*\bscale}
\draw[dashed, gecteal, line width=1.2pt]
    (\bstart - 0.6, \thY) -- (\bstart + 7*\bstep + 0.6, \thY);
\node[font=\scriptsize, gecteal, anchor=west] at (\bstart + 7*\bstep - 0.05, \thY + 0.16)
    {EMA Threshold $c_i$};

\draw[->, thick] (\bstart - 0.7, \bbase - 0.15) -- (\bstart - 0.7, \bbase + 4.5*\bscale + 0.3);
\node[font=\scriptsize, rotate=90, anchor=south] at (\bstart - 0.88, \bbase + 1.6) {Score $r_{t,i}$};

\draw[->, thick] (\bstart - 0.6, \bbase - 0.15) -- (\bstart + 7*\bstep + 0.7, \bbase - 0.15);
\node[font=\footnotesize] at (\bstart + 3.5*\bstep, \bbase - 0.45) {Token index $t$};

\node[font=\small, gecteal, align=center] at (\gecX, 0.96)
    {$z_{t,i} = \mathbf{1}\{r_{t,i} > c_i\}$};

\draw[gray!40, line width=1pt] (2.7, 0.6) -- (2.7, \titleY+0.4);
\draw[gray!40, line width=1pt] (10.8, 0.6) -- (10.8, \titleY+0.4);
\draw[gray!20, dashed] (\ecX+0.1, 0.6) -- (\ecX+0.1, \formulaY-0.2);

\def\axY{0.2}
\draw[-{Stealth[length=6pt]}, thick, gray!75] (-2.9, \axY) -- (16.8, \axY);
\node[font=\normalsize\itshape, gray!80] at (7.2, \axY - 0.58) {routing pool size};

\foreach \lab/\xp in {token/\tcX, sequence/\seqX, batch/\batX, population/\gecX} {
    \draw[thick, gray!80] (\xp, \axY-0.1) -- (\xp, \axY+0.1);
    \node[font=\footnotesize, gray!80, anchor=north] at (\xp, \axY-0.15) {\lab};
}

\end{tikzpicture}%
}
\caption{Illustration of TC, EC, and ET routing mechanisms and their routing pools. \textbf{Left:} TC routes each token independently to its top-$G$ experts, causing load imbalance. \textbf{Middle:} EC has each expert select its top-$k$ tokens from the batch, requiring access to all tokens including future ones (non-causal). \textbf{Right:} ET routes each token independently by comparing its score against the population's top-$(1/E)$ quantile estimated by an EMA-tracked threshold $c_i$, enabling fully causal routing over the population.}
\label{fig:illustration_v2}
\end{figure*}
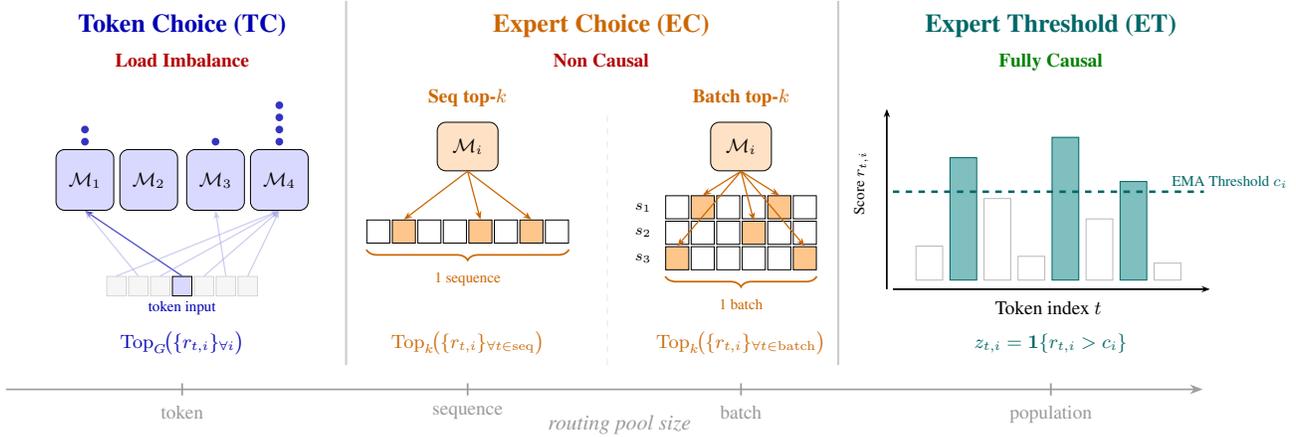

Mixture of Experts (MoE) architectures~\cite{shazeer2017outrageously, lepikhin2021gshard, fedus2022switch} have emerged as a leading approach to scale language models efficiently, powering frontier models like DeepSeek-V3~\cite{deepseekv3}. By sparsely activating only a subset of expert networks per token, MoE decouples model capacity from computational cost, enabling massive parameter counts with tractable FLOPs.
However, sparse routing introduces a fundamental tension: without intervention, routers tend to collapse onto a small subset of experts~\cite{shazeer2017outrageously}. This harms model quality, as underutilized experts become redundant parameters that waste capacity.
It also creates hardware bottlenecks under Expert Parallelism~\cite{lepikhin2021gshard}, where skewed loads leave some devices idle and others overloaded.
Thus, we need a routing mechanism that roughly maintains load balancing. 

Prior work falls into two categories.
The prevalent token choice (TC) routing~\cite{fedus2022switch} fixes the number of experts each token selects. This sparsity constraint not only fails to address load imbalance, but further complicates the routing as it conflicts with load balancing, turning the routing into a combinatorial optimization problem. People resort to heuristics to approximate load balancing, such as auxiliary losses~\cite{lepikhin2021gshard, fedus2022switch} or PID controllers~\cite{longcat2025, wang2024auxiliarylossfreeloadbalancingstrategy}.
In contrast, expert choice (EC) routing~\cite{zhou2022mixture} relaxes the fixed computation budget per token and only enforces load balancing within a batch by selecting the top-$k$ tokens for each expert, achieving perfect load balancing by construction while enabling dynamic computation allocation. However, EC routing fundamentally violates causality, making it unsuitable for autoregressive language models. Selecting top-$k$ requires comparing against the entire batch that includes future positions. At training time this mechanism leaks information~\cite{wang2024auxiliarylossfreeloadbalancingstrategy}; at inference time future tokens simply do not exist. 

In this paper, we relax both per-token sparsity and per-batch load balancing, requiring only that load reaches a targeted activation rate in expectation. The resulting mechanism, Expert Threshold (ET) routing, routes each token by comparing its score to a quantile threshold tracked from each expert's global score distribution. Because the same threshold is used at training and inference, ET routing is fully causal with no train-inference mismatch.

Pretraining a 2.4B (0.56B active) language model on FineWeb-Edu, ET outperforms TC by 0.067 in cross-entropy loss while achieving near-perfect load balancing. We further show that EC's performance improves with batch size, and that models trained with large-batch EC can perform causal inference using our threshold-based routing without retraining.

\section{Preliminaries: Routing as Constrained Optimization}
\label{sec:preliminaries}

An MoE layer replaces a dense feed-forward block with a router and $GE$ experts. Consider a batch of $N$ tokens with representations $x_t \in \mathbb{R}^d$. The router computes scores
\begin{equation}
r_{t,i} = (W_r x_t)_i,
\end{equation}
collected into a matrix $r \in \mathbb{R}^{N \times GE}$. Based on $r$, a routing rule produces a binary assignment $z \in \{0,1\}^{N \times GE}$ where $z_{t,i} = 1$ indicates expert $i$ is activated for token $t$ and $0$ otherwise. Each selected expert $i$ computes an output $y_{i,t} \in \mathbb{R}^d$, weighted by a gate value $p_{t,i} = \sigma(r_{t,i})$. The MoE output for token $t$ is
\begin{equation}
y_t = \sum_{i=1}^{GE} z_{t,i}\, p_{t,i}\, y_{i,t}.
\end{equation}
The routing rule that determines $z$ therefore controls both compute allocation and expert load balance. We formalize MoE routing as finding $z$ that maximizes the total routing score subject to computational constraints, since higher scores indicate stronger token-expert affinity and, through the gate $p_{t,i}$, larger expert contributions to the output.

\paragraph{Token Choice Routing}
The standard Token Choice routing goal is:
\begin{equation}
\label{eq:primal}
\begin{aligned}
\max_{z} \quad & \sum_{t=1}^N \sum_{i=1}^{GE} z_{t,i} r_{t,i} \\
\text{s.t.} \quad & \sum_{i=1}^{GE} z_{t,i} = G, \quad \forall t \quad \text{(Sparsity)} \\
& \sum_{t=1}^N z_{t,i} = k, \quad \forall i \quad \text{(Load Balancing)} \\
& z_{t,i} \in \{0, 1\}
\end{aligned}
\end{equation}
Here the sparsity constraint ensures each token selects exactly $G$ experts, and the Load Balancing constraint ensures each expert processes exactly $k = N/E$ tokens.
Solving~\eqref{eq:primal} exactly requires combinatorial algorithms such as the $O(N^3)$ Hungarian Matching algorithm. Most Token Choice (TC) methods therefore strictly enforce the sparsity constraint by setting $z_{t,i} = 1 \iff i \in \text{Top}_G(r_{t, \cdot})$, while relying on auxiliary losses~\cite{lepikhin2021gshard, fedus2022switch} or loss-free load balancing strategies~\cite{wang2024auxiliarylossfreeloadbalancingstrategy} to approximate the load balancing constraint.

\paragraph{Expert Choice Routing}
While the load balancing constraint is essential to avoid routing collapse, the sparsity constraint has no practical benefit. Thus, Expert Choice (EC)~\cite{zhou2022mixture} removes the sparsity constraint entirely and enforces only load balancing within batches. The primal problem becomes:
\begin{equation}
\label{eq:ec_relax}
\begin{aligned}
\max_{z} \quad & \sum_{t=1}^N \sum_{i=1}^{GE} z_{t,i} r_{t,i} \\
\text{s.t.} \quad & \sum_{t=1}^N z_{t,i} = k, \quad \forall i \\
& z_{t,i} \in \{0, 1\}
\end{aligned}
\end{equation}
with trivial closed-form solution $z_{t,i} = \mathbf{1}\{t \in \text{Top}_k(r_{\cdot, i})\}$, i.e. picking the top-$k$ tokens in each batch. This design has two key benefits: (1) Perfect load balancing: each expert processes exactly $k = N/E$ tokens by construction, eliminating the need for auxiliary losses or capacity clipping; (2) Dynamic computation: a token may be selected by zero, one, or multiple experts, enabling adaptive compute allocation based on token importance.

However, the per sequence load balancing constraint in EC introduces a causality problem for autoregressive generation. The selection indicator $z_{t,i}$ depends on all tokens' scores $\{r_{1,i}, \ldots, r_{N,i}\}$---including future tokens unavailable during inference. Extending EC to batch-level top-$k$~\cite{ludziejewski2024scaling} partially alleviates this but does not fully restore causality, as routing still depends on batch composition.

\section{Expert Threshold}
In the preliminaries, we identified the constraints that token choice and expert choice routing impose, yet we question their necessity. To avoid routing collapse, asymptotic load balancing suffices. ET further relaxes the per-sequence or per-batch Load Balancing constraint to a stochastic expectation:
\begin{equation}
\label{eq:gec_relax}
\begin{aligned}
\max_{z} \quad & \mathbb{E}_{\text{data}} \!\left[ \sum_{i=1}^{GE} z_{t,i} r_{t,i} \right] \\
\text{s.t.} \quad & \mathbb{E}_{\text{data}} [ z_{t,i} ] = \frac{1}{E}, \quad \forall i \\
& z_{t,i} \in \{0, 1\}
\end{aligned}
\end{equation}
Essentially, solving this primal problem is equivalent to picking the top $1/E$ fraction of tokens from the full router logit distribution, rather than from a single batch. We may obtain a $(1 - 1/E)$-quantile estimate $c_i$ via exponential moving average (EMA) of the $k$-th largest router logit of each batch. Then, for both training and inference, we route tokens via binary thresholding, setting 
\begin{equation}z_{t,i}=\mathbf{1}\{r_{t,i}>c_i\}\end{equation}
where $z_{t,i}\in \{0,1\}$ is the binary indicator of whether token $t$ is routed to expert $i$. Since $z_{t,i}$ depends only on $r_{t,i}$ and the global threshold $c_i$, routing is fully causal while satisfying load balancing in expectation.

\paragraph{Connection to EC.} 
Conceptually, ET can be viewed as expert choice routing over an infinitely large batch. In standard EC, each expert selects its top-$k$ tokens within the batch, so the selection threshold depends on all tokens present. As the batch size grows, however, each individual token's influence on this threshold vanishes, and the routing decision for any token becomes independent of others. ET approximates this limit by maintaining a fixed threshold estimated from the global token distribution.

ET and EC handle batch-wise variance differently. EC enforces perfect load balance per batch by letting the threshold vary, which means routing decisions fluctuate with batch composition. ET instead fixes the threshold for stable routing decisions, accepting small variance in per-batch expert utilization. Despite this difference in training, we show that ET routing can serve as causal inference for EC-trained models without retraining, provided the batch size is sufficiently large.

\paragraph{Warmup.} At the beginning of training, the router logits' distribution is not stable yet. The cutoff-EMA requires several thousand steps to converge to a meaningful estimate of the population quantile. During this period, incorrect thresholds cause severe expert starvation—most tokens fail to exceed the threshold, leaving experts underutilized. To address this cold-start problem, we use standard EC routing for the first 4k steps before switching to ET. This allows the cutoff-EMA to accumulate stable statistics under controlled load balance.


\begin{algorithm}[t]
    \caption{Expert Threshold Routing}
    \label{alg:gec}
    \begin{algorithmic}[1]
    \STATE \textbf{Input:} router logits $r \in \mathbb{R}^{N \times GE}$, cutoff-EMA $\{c_i\}$, decay rate $\beta$, target selection size $k= N/E$
    \FOR{expert $i = 1, \ldots, GE$}
        \STATE $z_{t,i} \gets \mathbf{1}\{r_{t,i} > c_i\} \quad \forall t$ 
        \IF{\textsc{Training}}
            \STATE $c_i \gets \beta c_i + (1-\beta)\cdot \text{kth-largest}(\{r_{t,i}\}_{t=1}^{N}, k)$
        \ENDIF
    \ENDFOR
    \STATE \textbf{Return} $z$, $\{c_i\}$
    \end{algorithmic}
    \end{algorithm}



\section{Experiments}
\label{sec:experiments}

\subsection{Experiment Setup}
We evaluate our methods on Nanochat~\citep{karpathy2025nanochat}, an open-source codebase for training GPT-2-like models. 
We conduct experiments at two scales: a d12 model (575M parameters, 195M active) with 12 transformer layers, and a d20 model (2.4B parameters, 561M active) with 20 transformer layers.
For MoE layers, we use 16 routed experts with granularity $G{=}1$ and expansion $E{=}16$, plus 1 shared expert. Each token activates the shared expert and on average 1 routed expert. We use sigmoid gates ($p_{t,i} = \sigma(r_{t,i})$) instead of softmax gates following LossFree~\cite{wang2024auxiliarylossfreeloadbalancingstrategy} and Mixture-of-Depths~\cite{raposo2024mixture}.
We add expert capacity factor of $C = 0.5$ to avoid GPU out-of-memory.
The first layer is kept dense following common practice~\citep{deepseekv3, wang2024auxiliarylossfreeloadbalancingstrategy} to allow meaningful routing. We train on 10B and 11.2B tokens for d12 and d20, respectively, from the FineWeb-Edu 100B dataset~\citep{penedo2024the} with a batch size of 0.5M tokens (for d20, we halve the minibatch size and use 2-step gradient accumulation). We report CE loss and CORE benchmark results~\citep{li2024dclm}. Architecture, training, and evaluation details are in Appendices~\ref{app:arch_diff}, \ref{app:training_setup}, and~\ref{app:coreeval}.

\subsection{Main Results}
We compare Expert Threshold (ET) routing against Expert Choice (EC) and Token Choice (TC) routing. All variants share the same architecture and parameter count. For ET, we use EMA decay $\beta = 0.999$ and EC warmup for the first 4k steps.
For EC, we sweep the global selection batch size from 2k to 512k tokens during training and use ET's cutoff EMA during inference which makes it fully causal. Unless stated, reported CORE/CE use the causal protocol.
For TC, we report variants with no load balancing, auxiliary loss ($\alpha{=}0.001$), and loss-free load balancing ($u{=}0.005$). Tables~\ref{tab:main_results} and~\ref{tab:d20_results} summarize results.
ET consistently outperforms TC in both CE loss (by 0.05 on d12 and 0.067 on d20) and CORE (by 1.89 on d12 and 2.83 on d20). 
EC with large batch sizes achieves comparable CE loss to ET, confirming that explicit large-batch selection and EMA-based thresholding reach similar training loss. EC 512k slightly edges out ET on CORE (19.94 vs.\ 19.88) in d12, though both substantially outperform TC.

\begin{table}[!t]
\centering
\caption{Main results comparing Expert Choice (EC), Token Choice (TC), and Expert Threshold (ET) routing. \textbf{Batch}: token routing pool size. EC uses global selection batch, TC uses per-step batch, ET reports effective EMA pool size $N/(1-\beta)$. TC variants: no load-balancing, auxiliary loss ($\alpha{=}0.001$), or loss-free ($u{=}0.005$). We report validation cross-entropy (CE) loss ($\downarrow$) and CORE Eval score ($\uparrow$).}
\label{tab:main_results}
\small
\setlength{\tabcolsep}{4pt}
\begin{tabularx}{\columnwidth}{@{}>{\raggedright\arraybackslash}X c r r@{}}
\toprule
\textbf{Method} & \textbf{Batch} & \textbf{CE loss ($\downarrow$)} & \textbf{CORE ($\uparrow$)} \\
\midrule
dense & --- & 3.002 & 15.743 \\
TC  & --- &  2.893 & 17.983 \\
TC aux & 64k & 2.892 & 15.894 \\
TC loss-free & 512k & 2.898 & 18.031 \\
\midrule
EC & 2k & 2.910 & 17.91 \\
EC & 8k & 2.845 & 18.83 \\
EC & 64k & \textbf{2.841} & 18.754 \\
EC & 512k & \textit{2.843} & \textbf{19.94} \\
\midrule
ET ($\beta{=}0.999$+warmup) & 0.5M $\rightarrow$ 500M & 2.844 & \textit{19.876} \\
\bottomrule
\end{tabularx}
\end{table}

\begin{table}[t]
\centering
\caption{d20 results. }
\label{tab:d20_results}
\small
\setlength{\tabcolsep}{4pt}
\begin{tabularx}{\columnwidth}{@{}>{\raggedright\arraybackslash}X c r r@{}}
\toprule
\textbf{Method} & \textbf{Batch} & \textbf{CE loss ($\downarrow$)} & \textbf{CORE ($\uparrow$)} \\
\midrule
dense & --- & 2.751 & 20.43 \\
TC aux & 32k & 2.687 & 22.31 \\
EC & 256k & \textit{2.621} & \textit{24.98} \\
ET ($\beta{=}0.999$+warmup) & 256k $\rightarrow$ 500M & \textbf{2.620} & \textbf{25.14} \\
\bottomrule
\end{tabularx}
\end{table}

\subsection{Analysis}
\label{sec:analysis}

We analyze key aspects of Expert Threshold routing through cutoff-usage tradeoff, dynamic computation allocation, expert specialization, and supporting EC comparisons on batch size scaling and train-evaluation gap.

\subsubsection{Cutoff vs Expert Usage Tradeoff}
\label{sec:cutoff_usage}

\begin{figure}[!tbp]
    \centering
    \includegraphics[width=\linewidth]{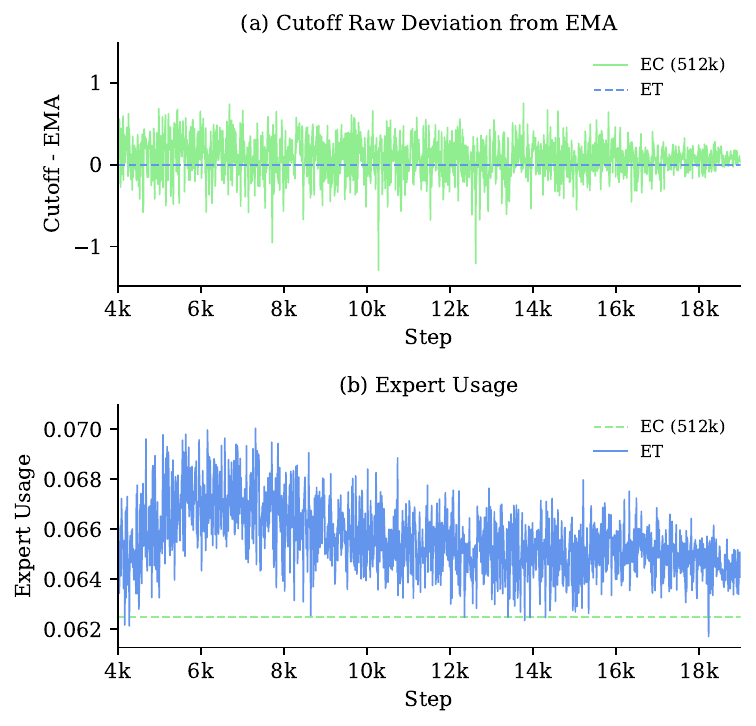}
    \caption{Cutoff stability vs expert usage tradeoff. \textbf{Top} Signed cutoff deviation relative to the EMA for EC at 512k batch size. ET stays at zero because routing uses the cutoff EMA directly. \textbf{Bottom} Expert usage for EC at 512k and ET. ET varies around the capacity target while EC remains constant. }
    \label{fig:cutoff_vs_usage}
\end{figure}

\begin{figure*}[!t]
    \centering
    \begin{minipage}{0.38\textwidth}
        \centering
        \includegraphics[width=\textwidth]{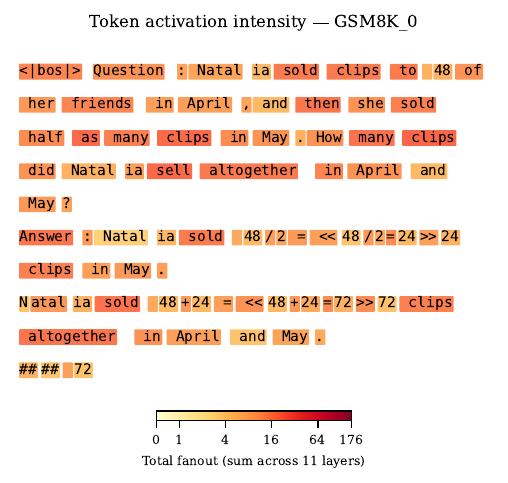}
        \subcaption{Per-token expert routing on a GSM8K passage.}
    \end{minipage}
    \hfill
    \begin{minipage}{0.58\textwidth}
        \centering
        \includegraphics[width=\textwidth]{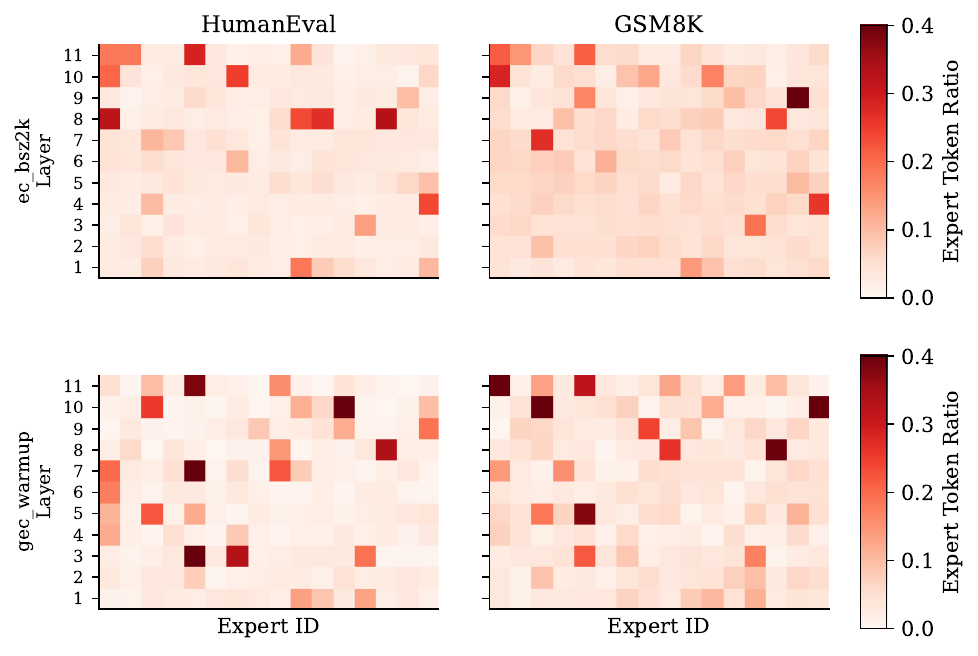}
        \subcaption{Expert activation heatmap. \textbf{Top:} EC with batch size 2k shows less specialization. \textbf{Bottom:} ET shows more extreme activation patterns, suggesting more domain-aware routing.}
    \end{minipage}
    \caption{Expert specialization analysis. (a)~Token-level activation intensity on a GSM8K passage, colored by total fanout (sum of experts activated across layers). The model assigns more computation to structurally important tokens (punctuation, sentence boundaries, numerical results) than to common content words. (b)~Expert token ratio heatmaps for HumanEval (code) and GSM8K (math). \textbf{Top:} EC (batch size 2k). \textbf{Bottom:} ET. ET achieves sharper patterns in expert activation, suggesting more domain-aware routing and specialization.}
    \label{fig:expert_specialization}
\end{figure*}

EC and ET achieve routing stability through complementary mechanisms. EC enforces a fixed expert usage: each expert selects exactly top-$k$ tokens, guaranteeing usage of $1/E$ per expert. However, the cutoff threshold varies batch-to-batch, with standard deviation scaling as $\mathcal{O}(1/\sqrt{N})$. ET inverts this tradeoff. The cutoff-EMA provides a stable threshold ($\beta = 0.999$), while expert usage fluctuates around the capacity target. Figure~\ref{fig:cutoff_vs_usage} shows the signed deviation between EC's per batch cutoff and cutoff-EMA, while ET remains at zero by design. This enables consistent inference without large-batch coordination. In essence, ET trades off hardware consistency for training-inference uniformity.

\subsubsection{Dynamic Computation Allocation}
\label{sec:dynamic_computation_allocation}

A key advantage of ET and EC is that they do not enforce a fixed amount of computation for every token. We here document its behavior and compare it with EC. For a more drastic comparison, we use the sequence-level EC with batch size 2k. Figure~\ref{fig:expert_specialization}(a) gives a qualitative example on a GSM8K passage~\citep{cobbe2021training}, where total fanout highlights tokens that receive heavier computation.

We further analyze how expert activation relates to position and token difficulty. Figure~\ref{fig_activation_dynamics_main} shows that both methods allocate more computation to early positions, but EC (2k) exhibits a dramatic spike at the first token (mean fanout $\sim$10) while ET shows a milder increase ($\sim$2) that decays smoothly. The lower row bins tokens by loss and overlays faint dashed layer traces with a denser global trend. For EC (2k), both the global curve and several layers rise with loss, showing that harder tokens receive more computation. ET remains flatter overall, with layer trajectories crossing and the global curve peaking in the middle before softening at higher loss. Additional layerwise views for the two main runs and extended comparisons for the remaining runs appear in Appendix~\ref{app_activation_dynamics_sweep}. 

\begin{figure*}[!t]
    \centering
    \begin{tabular}{@{}cc@{}}
        \includegraphics[width=0.48\textwidth]{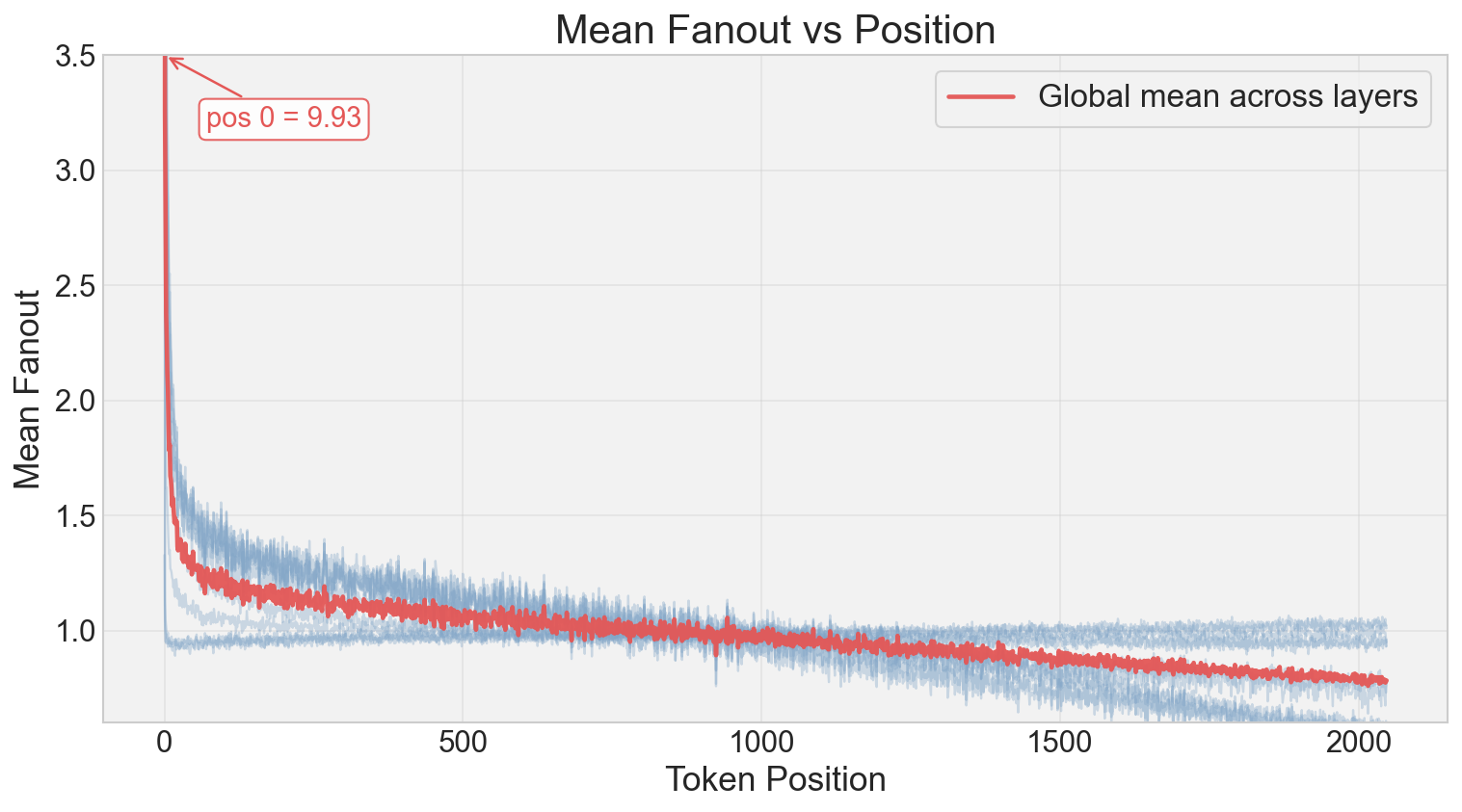} &
        \includegraphics[width=0.48\textwidth]{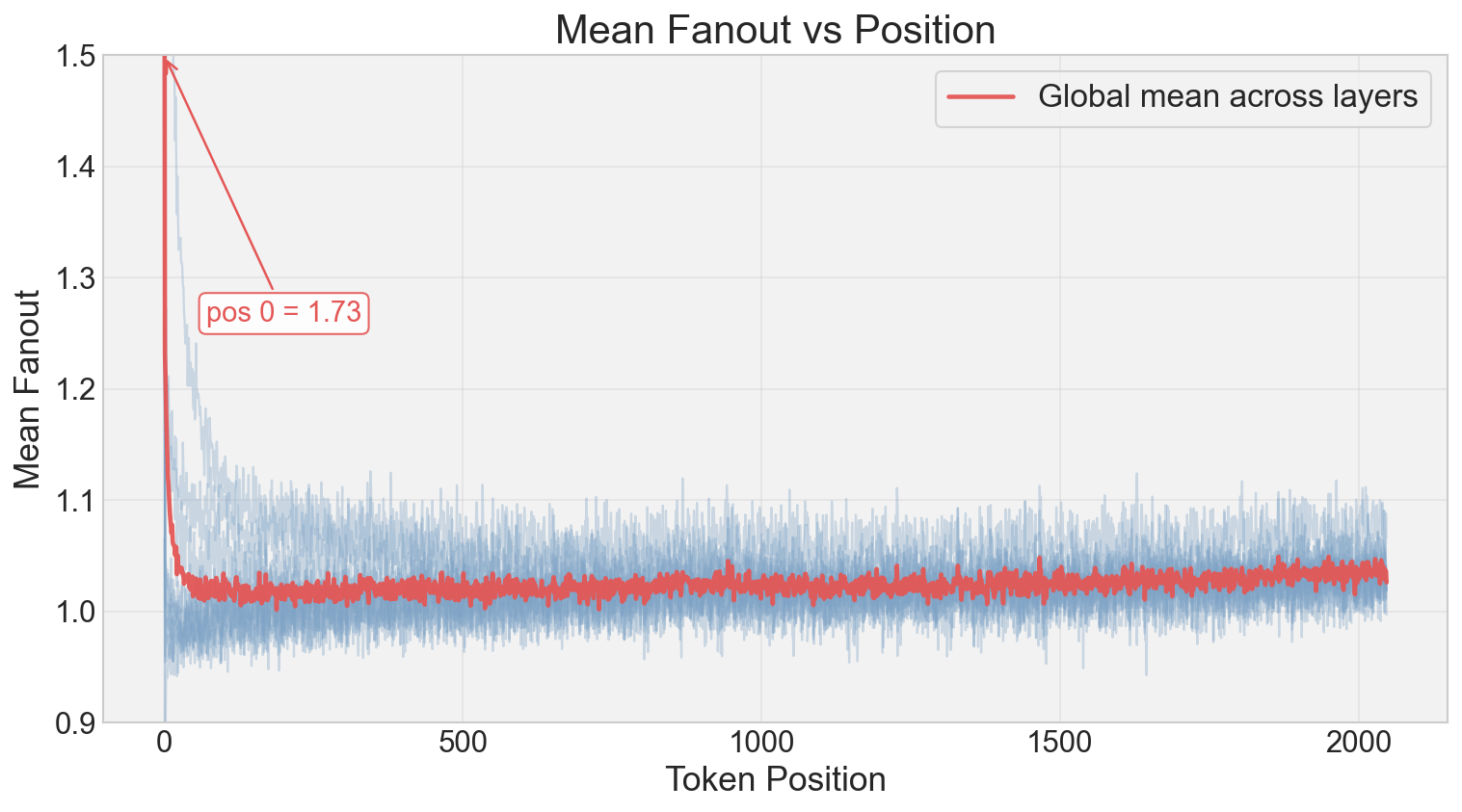} \\
        (a) EC (2k) fanout vs position & (b) ET fanout vs position \\[4pt]
        \includegraphics[width=0.48\textwidth]{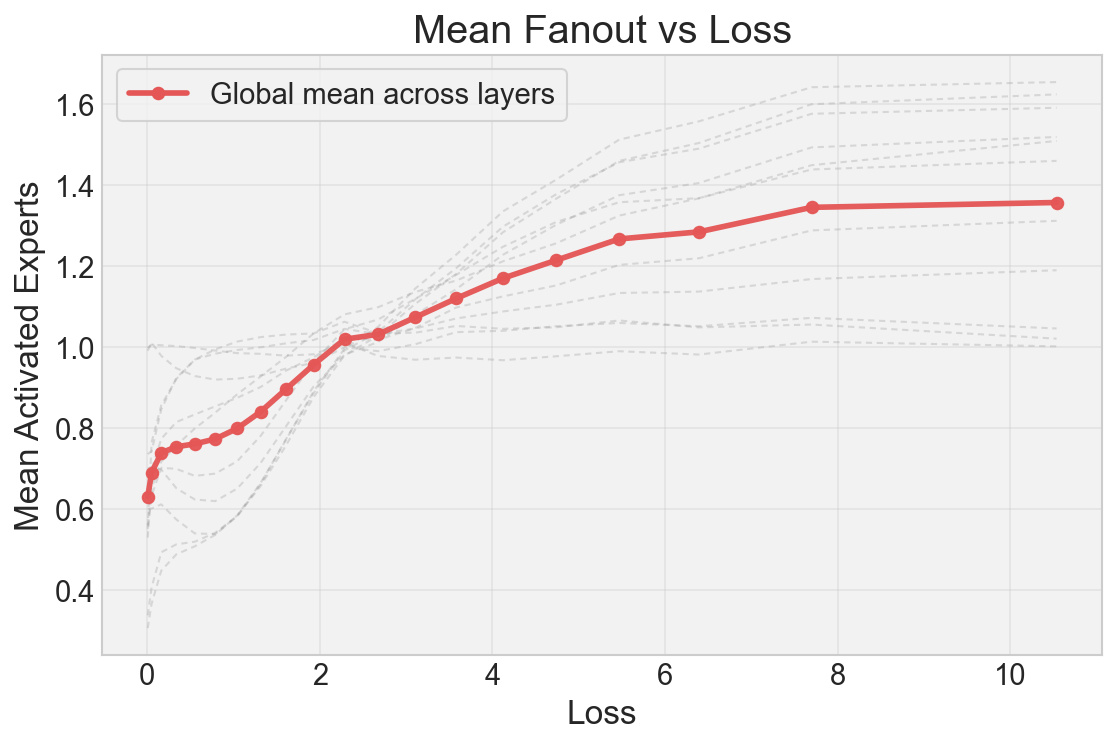} &
        \includegraphics[width=0.48\textwidth]{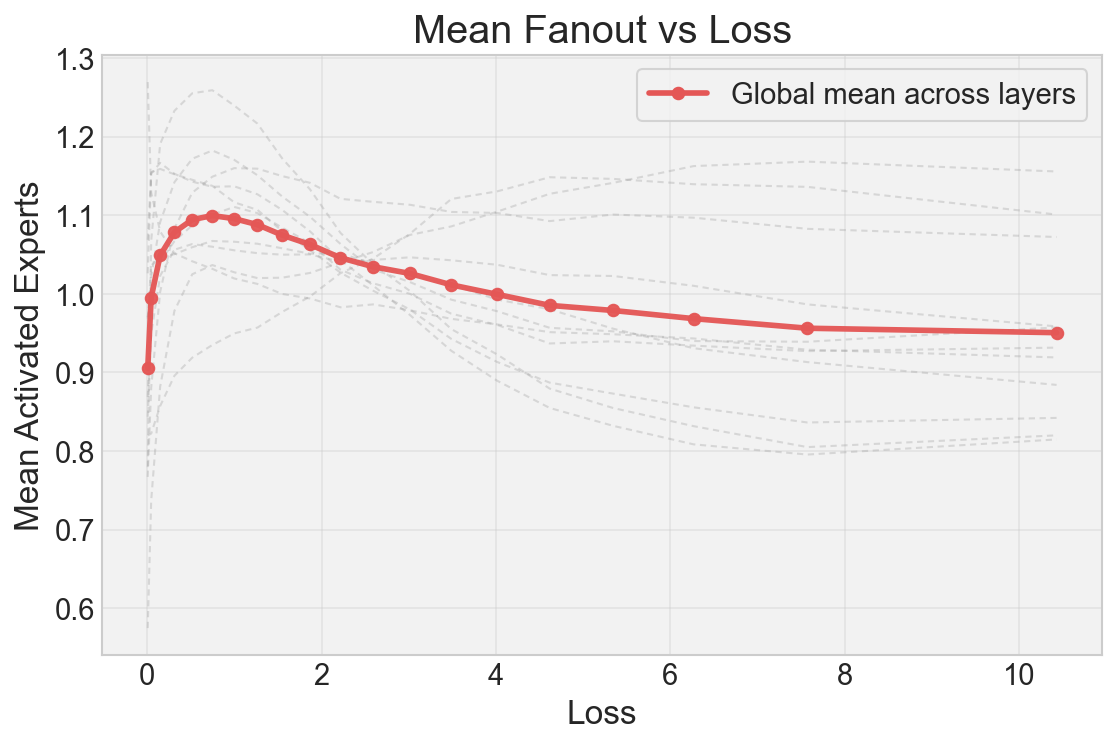} \\
        (c) EC (2k) fanout vs loss by layer & (d) ET fanout vs loss by layer \\
    \end{tabular}
    \caption{Activation dynamics for EC (2k) and ET. Both methods allocate more computation to early positions, with EC (2k) showing a sharper spike. Faint dashed curves show per-layer means and solid red curves show the global mean. When binned by loss, EC (2k) fanout rises monotonically while ET peaks early before declining.}
    \label{fig_activation_dynamics_main}
\end{figure*}

\subsubsection{Expert Specialization}
\label{sec:expert_specialization}

We follow Global LBL~\cite{qiu2025demonsdetailimplementingload} to evaluate expert specialization across EC with various batch sizes (2k, 8k, 64k, 512k) and ET. For each configuration, we measure the expert token ratio---the fraction of tokens from a given domain routed to each expert---across HumanEval~\citep{chen2021evaluating} (code) and GSM8K~\citep{cobbe2021training} (math) evaluation sets. Figure~\ref{fig:expert_specialization}(b) compares EC (batch size 2k) with ET. Both exhibit clear specialization: certain experts consistently attract domain-specific tokens, visible as concentrated dark cells in the heatmap. ET achieves specialization comparable to EC without requiring large-batch coordination at inference. The full comparison across all batch sizes (Appendix~\ref{app:additional_results}, Figure~\ref{fig:heatmap_combined}) shows that EC specialization sharpens with larger batches---patterns become more concentrated from 2k to 512k---while ET matches the large-batch EC pattern.

\subsubsection{Batch Size Scaling}
\label{sec:batch_size_scaling}

We hypothesize that larger batch sizes stabilize EC's cutoff threshold, yielding better performance and motivating ET's pursuit of the infinite-batch limit. Figure~\ref{fig:ec_bsz_scaling} confirms this trend across four batch sizes (2k, 8k, 64k, 512k tokens). Training CE loss improves from 2.874 (2k) to 2.844 (8k) to 2.836 (64k), with CORE Eval scores following suit (17.91 $\rightarrow$ 18.83 $\rightarrow$ 18.75). Top-$k$ selection over larger token pools better approximates the population-level routing decision, explaining this gain. However, performance saturates around 64k tokens, as increasing to 512k provides no further improvement (2.840 CE, 19.94 CORE Eval).

Figure~\ref{fig:ec_bsz_scaling} visualizes this scaling behavior. Notably, ET achieves comparable performance (2.844 CE, 19.876 CORE Eval) without requiring batch size coordination, making it practical for autoregressive inference where only single tokens are available.

\begin{figure}[!tbp]
    \centering
    \includegraphics[width=\linewidth]{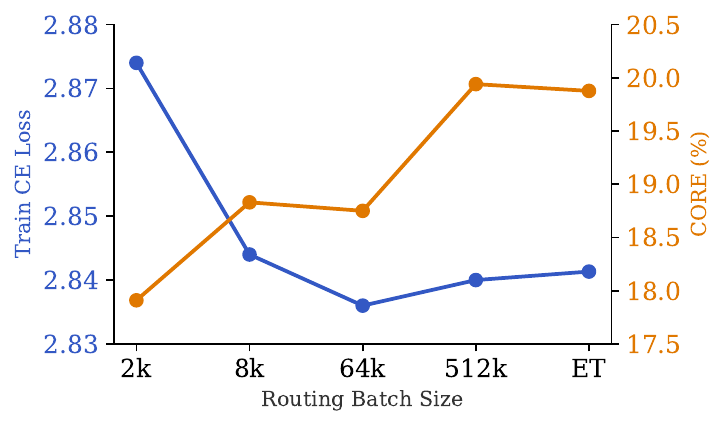}
    \caption{\color[HTML]{333333}EC performance across routing batch sizes. \textcolor[HTML]{3358C4}{Training CE loss} decreases and \textcolor[HTML]{E07800}{CORE Eval score} increases with larger batches. }
    \label{fig:ec_bsz_scaling}
\end{figure}

\subsubsection{Train-Evaluation Gap}
\label{sec:train_inference}

A key concern for Expert Choice is the train-inference discrepancy when using ET routing at inference. During training, EC selects the top-$k$ tokens for each expert within a batch; at inference, we apply ET's learned thresholds instead, since future tokens are unavailable for batch-level selection.

Our results demonstrate that this concern depends critically on the routing batch size. As shown in Table~\ref{tab:main_results}, EC with large batch sizes (64k, 512k) achieves validation loss nearly identical to ET (2.841--2.843 vs 2.844), with comparable CORE Eval scores. However, \textbf{smaller batch sizes reveal significant train-inference mismatch}: EC at 2k tokens shows degraded CORE Eval performance (17.91 vs 19.94 at 512k) and evaluation loss (2.910 vs 2.843). This gap arises because top-$k$ selection over a small batch is a noisy estimate of the population-level routing decision; at inference (batch size 1), this noise becomes extreme.

Figure~\ref{fig:train_vs_eval} illustrates this gap. EC (2k) shows a large train-evaluation discrepancy, while EC (512k) maintains close alignment between train loss EMA and eval loss. ET's cutoff-EMA mechanism addresses this by maintaining a population-level threshold that is independent of batch size, enabling consistent routing at inference without large-batch coordination.

\begin{figure}[!tbp]
    \centering
    \includegraphics[width=\linewidth]{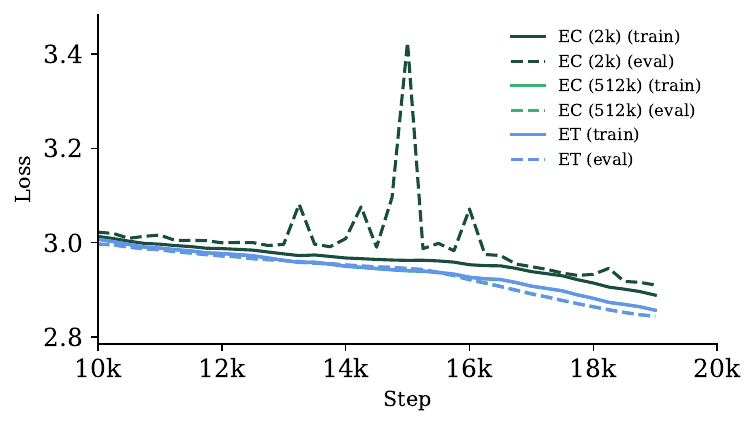}
    \caption{Train loss EMA and eval loss for EC at different batch sizes and ET. Solid lines show train loss EMA and dashed lines show eval loss. EC (2k) shows a large train-eval discrepancy, while EC (512k) and ET remain closely aligned.}
    \label{fig:train_vs_eval}
\end{figure}

\subsubsection{Routing Consistency Across Checkpoints}
\label{sec:routing_consistency}

To measure how stably each routing rule preserves token expert assignments over training, we compare the routed-expert sets assigned to the same token-layer pairs across checkpoints, excluding the always-active shared expert. We report weighted Jaccard over pooled token-layer-expert edges,
\[
\mathrm{weighted\_jaccard}
=
\frac{|E_A \cap E_B|}{|E_A \cup E_B|},
\]
where $E_A$ and $E_B$ are the pooled active token-layer-expert edges under two checkpoints. A higher weighted Jaccard indicates more similar routing behaviors between checkpoints. This gives the clearest separation while preserving the same qualitative ranking as the companion divergence views in Appendix~\ref{app:routing_consistency}.

\begin{figure*}[!t]
    \centering
    \includegraphics[width=\textwidth]{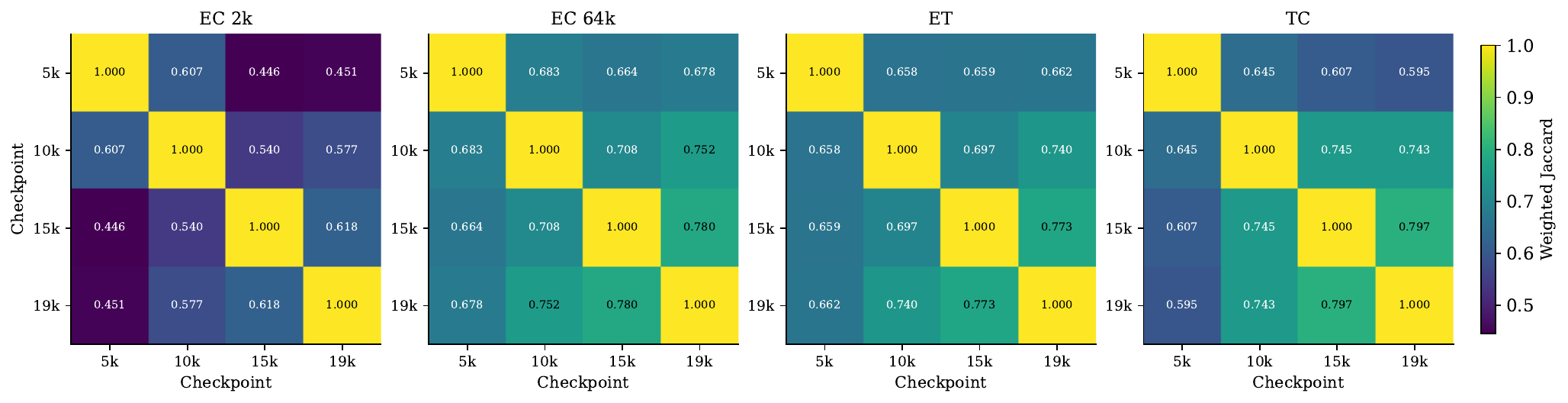}
    \caption{Within-family checkpoint-pair routing consistency on a fixed validation stream, measured by weighted Jaccard. ET is consistently more stable than EC 2k and stays close to EC 64k. TC is competitive on nearby checkpoints but degrades more on longer ranges.}
    \label{fig:routing_consistency_weighted_jaccard}
\end{figure*}

Figure~\ref{fig:routing_consistency_weighted_jaccard} shows a clear pattern. ET is above EC 2k on every checkpoint pair, indicating that threshold routing preserves its token-expert decisions much more consistently than small-pool EC. At the same time, ET remains close to EC 64k across the full matrix, which supports the view that ET tracks the large-pool EC regime without requiring large-batch coordination at inference. TC shows strong short-range consistency, but its longest-range pairs are weaker than ET, so it does not match the same large-pool EC behavior as cleanly. Appendix~\ref{app:routing_consistency} reports the complementary joint JSD heatmap.



\section{Related Work}

\subsection{Mixture of Experts}

Mixture of Experts (MoE) scales model capacity by routing each token to a small subset of experts while keeping compute nearly constant. A learned gate selects top-$G$ experts per token~\cite{shazeer2017outrageously}, with auxiliary losses to balance load across experts~\cite{lepikhin2021gshard}. The Switch Transformer~\cite{fedus2022switch} sets $G{=}1$ for efficiency. Recent LLMs further adopt fine-grained MoE with many small experts and shared experts that remain always active to capture global knowledge~\cite{dai2024deepseek}. We incorporate shared experts in our design.

\subsection{Load Balancing}

A critical challenge in MoE systems is load balancing, as routers often favor a small subset of experts without explicit constraints. The standard approach uses an auxiliary loss $\mathcal{L}_{\text{aux}} = \alpha \sum_{i} f_i P_i$ to encourage uniform expert assignment~\cite{lepikhin2021gshard, fedus2022switch}, where $f_i = \frac{E}{N}\sum_{t=1}^N z_{t,i}$ and $P_i = \frac{1}{N}\sum_{t=1}^N p_{t,i}$ are the normalized load and average routing probability for expert $i$. Minimizing this loss exerts unbalanced pressure to suppress the router logits based on the load statistics, which makes the router logits biased towards the less loaded experts.
However, in distributed training, small local batch sizes cause high variance in load estimation. Global-batch load balancing~\cite{qiu2025demonsdetailimplementingload,qwen3} addresses this by computing balance statistics across all devices, yielding more stable gradients and improved expert specialization. This insight motivates our approach to extend the ``global'' philosophy beyond auxiliary losses.


Recent work explores auxiliary-loss-free alternatives. DeepSeekMoE~\cite{dai2024deepseek} introduces expert-specific bias terms $b_i$ that dynamically adjust based on load statistics. Expert selection uses biased scores $r_{t,i} + b_i$, while gating weights use original scores $r_{t,i}$, preserving specialization. The bias updates follow $b_i \leftarrow b_i + u \cdot \text{sign}(1 - f_i)$, where $f_i$ is a normalized load statistic for expert $i$ (equal to 1 under perfect balance). This eliminates the trade-off between load balancing and task performance inherent in auxiliary loss methods. LongCat-Flash~\cite{longcat2025} adopts a similar framework but replaces the sign-based update with proportional control: $\Delta b_i = u \cdot (1 - f_i)$. While DeepSeek's approach applies constant-magnitude corrections regardless of imbalance severity, proportional updates scale with the load deviation, enabling smoother convergence. 

Expert Threshold (ET) combines the above ideas. Instead of a per-batch top-$k$ selection for the original EC, we extend Qwen's philosophy to compute balance statistics across the entire pretrain population by maintaining a distributional cutoff threshold using EMA. Such number, surprisingly, functions similarly to the bias term for loss-free load balancing. See Table~\ref{tab:conceptual_mapping} for more details.

\begin{table}[t]
\centering
\caption{Taxonomy of load balancing methods by scope (Aux loss~\cite{lepikhin2021gshard}; Global LBL~\cite{qiu2025demonsdetailimplementingload}; LossFree~\cite{wang2024auxiliarylossfreeloadbalancingstrategy}; Seq EC~\cite{zhou2022mixture}; Batch EC~\cite{ludziejewski2024scaling}).}
\label{tab:load_balance_taxonomy}
\small
\begin{tabularx}{\columnwidth}{@{}>{\centering\arraybackslash}X>{\centering\arraybackslash}X>{\centering\arraybackslash}X@{}}
\toprule
\textbf{Micro Batch/Seq} & \textbf{Batch} & \textbf{Population} \\
\midrule
Aux loss & Global LBL & -- \\
-- & -- & LossFree \\
Seq EC & Batch EC & ET (ours) \\
\bottomrule
\end{tabularx}
\end{table}

\begin{table}[t]
\centering
\caption{Conceptual connections between ET and recent work (LossFree~\cite{wang2024auxiliarylossfreeloadbalancingstrategy}; GShard~\cite{fedus2022switch}).}
\label{tab:conceptual_mapping}
\small
\begin{tabularx}{\columnwidth}{@{}>{\raggedright\arraybackslash}X>{\raggedright\arraybackslash}X>{\raggedright\arraybackslash}X@{}}
\toprule
\textbf{ET} & \textbf{Similar To} & \textbf{Connection} \\
\midrule
Cutoff-EMA $c_i$ & LossFree bias $b_i$ & Per-expert scalar; no aux loss \\
\midrule
$1 - \beta$ & LossFree $\mu$ & Update rate \\
\bottomrule
\end{tabularx}
\end{table}

\subsection{Dynamic Computation}

Dynamic computation methods adaptively allocate computational resources based on input complexity. Expert Choice (EC)~\cite{zhou2022mixture}, detailed in Section~\ref{sec:preliminaries}, achieves this by letting each expert select its top-$k$ tokens, enabling variable computation per token (0 to $GE$ experts). EC has been applied to upcycling dense checkpoints~\cite{komatsuzaki2023sparse}, attention layer skipping~\cite{raposo2024mixture}, vision~\cite{liu2024routersvision}, diffusion~\cite{sun2024ecdit, shi2025diffmoe}, and multimodal models~\cite{lin2024moma, ni2025openmoe2}. Related variants expand the design space~\cite{yan2025tcmoe}. However, EC's causality problem limits its use in autoregressive LLMs (Section~\ref{subsec:causal_generation_ec}).

Besides EC, other approaches to dynamic computation rely on other explicit designs. ReMoE~\cite{wang2025remoe} replaces discrete TopG routing with fully differentiable ReLU-based routing and adaptive L1 regularization.
Other works~\cite{jin2024moe++, longcat2025, zeng2024adamoe} introduce zero-computation experts (e.g., zero, copy, and constant) that allow tokens to skip expert computation entirely, an approach \citet{kilian2026datasparsity} extend to multimodal modeling.
Top-P routing~\cite{liu2024unimoeaudio, jin2025dtop_p, huang2024hardertasksneedexperts, wang2025hmoe} selects experts based on cumulative probability mass, adapting expert count to routing confidence, so high-confidence tokens use fewer experts while uncertain ones activate more.
XMoE~\cite{yang2024xmoe} is closest to our setting, replacing fixed Top-$G$ routing with a threshold that activates experts until cumulative routing mass exceeds a preset value. The key difference is that XMoE uses a fixed probability-mass threshold in token-choice MoE, while ET uses expert-specific EMA cutoffs to causalize expert choice. Auto-tuning methods like DynMoE~\cite{guo2025dynamicmixtureexpertsautotuning} also let each token determine how many experts to activate while reducing sensitivity to MoE hyperparameters. Beyond MoE routing itself, conditional computation can also be applied to other Transformer components and long-context settings, e.g., CoLT5~\cite{ainslie-etal-2023-colt5}.
Early exit methods~\cite{xin2020deebert} enable sample-level dynamics by allowing tokens to exit at intermediate layers.

\subsection{Causal Generation of Expert Choice Models}
\label{subsec:causal_generation_ec}

EC poses a causality challenge: token selection requires ranking against future tokens, which are unavailable in autoregressive generation.
Prior work addresses this issue in three main ways. Predictor-based methods train an auxiliary predictor or learn per-expert thresholds to approximate oracle top-$k$ decisions, enabling causal routing at inference~\cite{raposo2024mixture, shi2025diffmoe}. Alternatively, top-$k$ selection across the current tokens from different sequences preserves causality within each sequence~\cite{ludziejewski2024scaling, wen2025seqtopk}. Recent work changes routing granularity: Lory routes at the segment level, using the previous segment to determine the next~\cite{zhong2024lory}, while SeqTopK shifts expert budgets to sequence-level selection with an \emph{Expert Cache} for autoregressive decoding~\cite{wen2025seqtopk}.
All above approaches have significant drawbacks: predictions can be noisy and unstable, and batch-level top-$k$ can impose inference-time topology constraints, leading to a large train--inference mismatch; moreover, routing that depends on global batch composition can be sensitive to batch size/composition and raises privacy/safety concerns in multi-tenant settings~\cite{wen2025seqtopk}. In contrast to EC, ET reduces to a simple threshold test (whether token logit $r_{t,i}$ is higher than cutoff EMA $c_i$) at inference time, thus eliminating the train--inference discrepancy.

\subsection{From Batch to Population Level Statistics}

The progression from sample, batch, to population-level statistics is a recurring theme in deep learning. While techniques like Batch Normalization~\cite{ioffe2015batch} and contrastive learning~\cite{radford2021learning} rely on batch statistics, momentum-based approaches~\cite{he2020momentum,caron2021emerging} and adaptive optimizers like Adam~\cite{kingma2015adam} use Exponential Moving Averages (EMA) to approximate population distributions. 
ET applies this principle to routing via EMA-based cutoffs.

\section{Conclusion}

We introduce Expert Threshold (ET) routing, a mechanism that resolves the fundamental causality issue in Expert Choice (EC) models while preserving their load-balancing advantages. By maintaining an exponential moving average of each expert's selection threshold, estimated from historical batches rather than within-batch top-$k$ selection, ET routing enables fully causal routing. Each token's routing decision depends only on past statistics, eliminating the need for future token access at both training and inference time. Our experiments demonstrate that ET routing achieves competitive performance with EC routing (matching validation loss at 2.84) while outperforming Token Choice by 0.067 in cross-entropy loss, all while enabling causal autoregressive generation. The cutoff-EMA mechanism provides stable routing thresholds that accurately approximate EC's top-$k$ boundaries, as evidenced by the minimal train-inference gap observed across all metrics. We further show that a warmup strategy, using EC routing before transitioning to threshold-based selection, stabilizes early training dynamics. These findings suggest that the perceived incompatibility between Expert Choice routing and causal language modeling can be effectively bridged through population-level threshold estimation, opening new directions for scalable MoE architectures.

\section*{Acknowledgments}

We gratefully acknowledge the support of NVIDIA Corporation and the NVIDIA AI Technology Center (NVAITC) UF program. We thank Hongwu Peng for the generous support and guidance on development of the code. 

\section*{Impact Statement}

This paper presents work whose goal is to advance the field of Machine Learning. There are many potential societal consequences of our work, none which we feel must be specifically highlighted here.

\bibliography{example_paper}

@article{li2024dclm,
  title={DataComp-LM: In Search of the Next Generation of Training Sets for Language Models},
  author={Li, Jeffrey and Fang, Alex and Smber, Georgios and Wortsman, Mitchell and Gadre, Samir Yitzhak and Schmidt, Ludwig and others},
  journal={arXiv preprint arXiv:2406.11794},
  year={2024}
}

@article{chen2021evaluating,
  title={Evaluating Large Language Models Trained on Code},
  author={Chen, Mark and Tworek, Jerry and Jun, Heewoo and Yuan, Qiming and Pinto, Henrique Ponde de Oliveira and Kaplan, Jared and Edwards, Harri and Burda, Yuri and Joseph, Nicholas and Brockman, Greg and others},
  journal={arXiv preprint arXiv:2107.03374},
  year={2021}
}

@article{cobbe2021training,
  title={Training Verifiers to Solve Math Word Problems},
  author={Cobbe, Karl and Kosaraju, Vineet and Bavarian, Mohammad and Chen, Mark and Jun, Heewoo and Kaiser, Lukasz and Plappert, Matthias and Tworek, Jerry and Hilton, Jacob and Nakano, Reiichiro and others},
  journal={arXiv preprint arXiv:2110.14168},
  year={2021}
}

@article{raposo2024mixture,
  title={Mixture-of-Depths: Dynamically allocating compute in transformer-based language models},
  author={Raposo, David and Ritter, Sam and Richards, Blake and Lillicrap, Timothy and Humphreys, Peter Conway and Santoro, Adam},
  journal={arXiv preprint arXiv:2404.02258},
  year={2024}
}

@article{zhou2022mixture,
  title={Mixture-of-experts with expert choice routing},
  author={Zhou, Yanqi and Lei, Tao and Liu, Hanxiao and Du, Nan and Huang, Yanping and Zhao, Vincent and Dai, Andrew M and Le, Quoc V and Laudon, James and others},
  journal={Advances in Neural Information Processing Systems},
  volume={35},
  pages={7103--7114},
  year={2022}
}

@inproceedings{ainslie-etal-2023-colt5,
    title = "{C}o{LT}5: Faster Long-Range Transformers with Conditional Computation",
    author = "Ainslie, Joshua  and
      Lei, Tao  and
      de Jong, Michiel  and
      Ontanon, Santiago  and
      Brahma, Siddhartha  and
      Zemlyanskiy, Yury  and
      Uthus, David  and
      Guo, Mandy  and
      Lee-Thorp, James  and
      Tay, Yi  and
      Sung, Yun-Hsuan  and
      Sanghai, Sumit",
    editor = "Bouamor, Houda  and
      Pino, Juan  and
      Bali, Kalika",
    booktitle = "Proceedings of the 2023 Conference on Empirical Methods in Natural Language Processing",
    month = dec,
    year = "2023",
    address = "Singapore",
    publisher = "Association for Computational Linguistics",
    url = "https://aclanthology.org/2023.emnlp-main.309/",
    doi = "10.18653/v1/2023.emnlp-main.309",
    pages = "5085--5100"
}

@misc{huang2024hardertasksneedexperts,
      title={Harder Tasks Need More Experts: Dynamic Routing in MoE Models},
      author={Quzhe Huang and Zhenwei An and Nan Zhuang and Mingxu Tao and Chen Zhang and Yang Jin and Kun Xu and Liwei Chen and Songfang Huang and Yansong Feng},
      year={2024},
      eprint={2403.07652},
      archivePrefix={arXiv},
      primaryClass={cs.LG},
      url={https://arxiv.org/abs/2403.07652},
}

@inproceedings{
wang2025remoe,
title={ReMoE: Fully Differentiable Mixture-of-Experts with Re{LU} Routing},
author={Ziteng Wang and Jun Zhu and Jianfei Chen},
booktitle={The Thirteenth International Conference on Learning Representations},
year={2025},
url={https://openreview.net/forum?id=4D0f16Vwc3}
}

@misc{guo2025dynamicmixtureexpertsautotuning,
      title={Dynamic Mixture of Experts: An Auto-Tuning Approach for Efficient Transformer Models},
      author={Yongxin Guo and Zhenglin Cheng and Xiaoying Tang and Zhaopeng Tu and Tao Lin},
      year={2025},
      eprint={2405.14297},
      archivePrefix={arXiv},
      primaryClass={cs.LG},
      url={https://arxiv.org/abs/2405.14297},
}

@misc{liu2024unimoeaudio,
      title={UniMoE-Audio: Unified Speech and Music Generation with Dynamic-Capacity MoE},
      author={Zhenyu Liu and Yunxin Li and Xuanyu Zhang and Qixun Teng and Shenyuan Jiang and Xinyu Chen and Haoyuan Shi and Jinchao Li and Qi Wang and Haolan Chen and Fanbo Meng and Mingjun Zhao and Yu Xu and Yancheng He and Baotian Hu and Min Zhang},
      year={2025},
      eprint={2510.13344},
      archivePrefix={arXiv},
      primaryClass={cs.SD},
      url={https://arxiv.org/abs/2510.13344},
}

@misc{ni2025openmoe2,
      title={OpenMoE 2: Sparse Diffusion Language Models},
      author={Ni, Jinjie and team},
      year={2025},
      howpublished={\url{https://github.com/JinjieNi/OpenMoE2}},
}

@article{shazeer2017outrageously,
  title={Outrageously large neural networks: The sparsely-gated mixture-of-experts layer},
  author={Shazeer, Noam and Mirhoseini, Azalia and Maziarz, Krzysztof and Davis, Andy and Le, Quoc and Hinton, Geoffrey and Dean, Jeff},
  journal={arXiv preprint arXiv:1701.06538},
  year={2017}
}

@inproceedings{
lepikhin2021gshard,
title={{\{}GS{\}}hard: Scaling Giant Models with Conditional Computation and Automatic Sharding},
author={Dmitry Lepikhin and HyoukJoong Lee and Yuanzhong Xu and Dehao Chen and Orhan Firat and Yanping Huang and Maxim Krikun and Noam Shazeer and Zhifeng Chen},
booktitle={International Conference on Learning Representations},
year={2021},
url={https://openreview.net/forum?id=qrwe7XHTmYb}
}

@article{fedus2022switch,
  title={Switch transformers: Scaling to trillion parameter models with simple and efficient sparsity},
  author={Fedus, William and Zoph, Barret and Shazeer, Noam},
  journal={Journal of Machine Learning Research},
  volume={23},
  number={120},
  pages={1--39},
  year={2022}
}

@inproceedings{komatsuzaki2023sparse,
  title={Sparse Upcycling: Training Mixture-of-Experts from Dense Checkpoints},
  author={Komatsuzaki, Aran and Puigcerver, Joan and Lee-Thorp, James and Ruiz, Carlos Riquelme and Mustafa, Basil and Ainslie, Joshua and Tay, Yi and Dehghani, Mostafa and Houlsby, Neil},
  booktitle={International Conference on Learning Representations (ICLR)},
  year={2023},
  url={https://arxiv.org/abs/2212.05055}
}

@inproceedings{ludziejewski2024scaling,
  title={Scaling Laws for Fine-Grained Mixture of Experts},
  author={Ludziejewski, Jakub and Krajewski, Jakub and Adamczewski, Kamil and Pioro, Maciej and Chowdhury, Siddhartha and Sanyal, Arnab and Miasojedow, Bartlomiej and Pontes, Haya R and Jaszczur, Sebastian and Pacek, Bart and Jastrz{\k{e}}bski, Stanis{\l}aw and Bousquet, Olivier and Hoogeboom, Emiel and Michalewski, Henryk},
  booktitle={Proceedings of the 41st International Conference on Machine Learning},
  pages={32790--32809},
  year={2024},
  volume={235},
  series={Proceedings of Machine Learning Research},
  url={https://proceedings.mlr.press/v235/ludziejewski24a.html}
}

@misc{wang2024auxiliarylossfreeloadbalancingstrategy,
      title={Auxiliary-Loss-Free Load Balancing Strategy for Mixture-of-Experts},
      author={Lean Wang and Huazuo Gao and Chenggang Zhao and Xu Sun and Damai Dai},
      year={2024},
      eprint={2408.15664},
      archivePrefix={arXiv},
      primaryClass={cs.LG},
      url={https://arxiv.org/abs/2408.15664},
}

@misc{dai2024deepseek,
  title={DeepSeekMoE: Towards Ultimate Expert Specialization in Mixture-of-Experts Language Models},
  author={Dai, Damai and Deng, Chengqi and Zhao, Chenggang and Xu, R. X. and Gao, Huazuo and Chen, Deli and Li, Jiashi and Zeng, Wangding and Yu, Xingkai and Wu, Y. and Xie, Zhenda and Li, Y. K. and Huang, Panpan and Luo, Fuli and Cheng, A. and Zhang, Kai and Sui, Jie and Zhao, Xiaokang and Xing, Nuo and Peng, Zhiyuan and Jie, Shaosheng and Yang, Tong and Gao, Wei and Wang, Qi and Zeng, Yushi and Gao, Chong and Xiong, Runji and Sun, Xu},
  year={2024},
  eprint={2401.06066},
  archivePrefix={arXiv},
  primaryClass={cs.CL},
  url={https://arxiv.org/abs/2401.06066}
}

@misc{deepseekv3,
  title={DeepSeek-V3 Technical Report},
  author={DeepSeek-AI},
  year={2024},
  eprint={2412.19437},
  archivePrefix={arXiv},
  primaryClass={cs.CL},
  url={https://arxiv.org/abs/2412.19437}
}

@misc{radford2019language,
  title={Language Models are Unsupervised Multitask Learners},
  author={Radford, Alec and Wu, Jeff and Child, Rewon and Luan, David and Amodei, Dario and Sutskever, Ilya},
  year={2019},
  howpublished={OpenAI Blog},
  url={https://openai.com/blog/better-language-models/}
}

@misc{karpathy2025nanochat,
  title={nanochat: The best ChatGPT that \$100 can buy},
  author={Karpathy, Andrej},
  year={2025},
  howpublished={\url{https://github.com/karpathy/nanochat}},
  note={GitHub repository}
}

@inproceedings{
penedo2024the,
title={The FineWeb Datasets: Decanting the Web for the Finest Text Data at Scale},
author={Guilherme Penedo and Hynek Kydl{\'\i}{\v{c}}ek and Loubna Ben allal and Anton Lozhkov and Margaret Mitchell and Colin Raffel and Leandro Von Werra and Thomas Wolf},
booktitle={The Thirty-eight Conference on Neural Information Processing Systems Datasets and Benchmarks Track},
year={2024},
url={https://openreview.net/forum?id=n6SCkn2QaG}
}

@misc{sun2024ecdit,
  title={EC-DIT: Scaling Diffusion Transformers with Adaptive Expert-Choice Routing},
  author={Sun, Haotian and Lei, Tao and Zhang, Bowen and Li, Yanghao and Huang, Haoshuo and Pang, Ruoming and Dai, Bo and Du, Nan},
  year={2024},
  eprint={2410.02098},
  archivePrefix={arXiv},
  primaryClass={cs.CV},
  url={https://arxiv.org/abs/2410.02098}
}

@misc{shi2025diffmoe,
  title={DiffMoE: Dynamic Token Selection for Scalable Diffusion Transformers},
  author={Shi, Minglei and Yuan, Ziyang and Yang, Haotian and Wang, Xintao and Zheng, Mingwu and Tao, Xin and Zhao, Wenliang and Zheng, Wenzhao and Zhou, Jie and Lu, Jiwen and Wan, Pengfei and Zhang, Di and Gai, Kun},
  year={2025},
  eprint={2503.14487},
  archivePrefix={arXiv},
  primaryClass={cs.CV},
  url={https://arxiv.org/abs/2503.14487}
}

@misc{lin2024moma,
  title={MoMa: Efficient Early-Fusion Pre-training with Mixture of Modality-Aware Experts},
  author={Lin, Xi Victoria and Shrivastava, Akshat and Luo, Liang and Iyer, Srinivasan and Lewis, Mike and Ghosh, Gargi and Zettlemoyer, Luke and Aghajanyan, Armen},
  year={2024},
  eprint={2407.21770},
  archivePrefix={arXiv},
  primaryClass={cs.AI},
  url={https://arxiv.org/abs/2407.21770}
}

@inproceedings{xin2020deebert,
  title={DeeBERT: Dynamic Early Exiting for Accelerating BERT Inference},
  author={Xin, Ji and Tang, Raphael and Lee, Jaejun and Yu, Yaoliang and Lin, Jimmy},
  booktitle={Proceedings of the 58th Annual Meeting of the Association for Computational Linguistics},
  pages={2246--2251},
  year={2020}
}

@misc{jin2024moe++,
  title={MoE++: Accelerating Mixture-of-Experts Methods with Zero-Computation Experts},
  author={Jin, Peng and Zhu, Bo and Yuan, Li and Yan, Shuicheng},
  year={2024},
  eprint={2410.07348},
  archivePrefix={arXiv},
  primaryClass={cs.LG},
  url={https://arxiv.org/abs/2410.07348}
}

@misc{kilian2026datasparsity,
  title={Improving MoE Compute Efficiency by Composing Weight and Data Sparsity},
  author={Kilian, Maciej and Mkrtchyan, Oleg and Zettlemoyer, Luke and Shrivastava, Akshat and Aghajanyan, Armen},
  year={2026},
  eprint={2601.15370},
  archivePrefix={arXiv},
  primaryClass={cs.LG},
  url={https://arxiv.org/abs/2601.15370}
}

@misc{rajbhandari2020zero,
  title={ZeRO: Memory Optimizations Toward Training Trillion Parameter Models},
  author={Rajbhandari, Samyam and Rasley, Jeff and Ruwase, Olatunji and He, Yuxiong},
  year={2020},
  eprint={1910.02054},
  archivePrefix={arXiv},
  primaryClass={cs.LG},
  url={https://arxiv.org/abs/1910.02054}
}

@misc{tan2024scattermoe,
  title={Scattered Mixture-of-Experts Implementation},
  author={Tan, Shawn and Shen, Yikang and Panda, Rameswar and Courville, Aaron},
  year={2024},
  eprint={2403.08245},
  archivePrefix={arXiv},
  primaryClass={cs.LG},
  url={https://arxiv.org/abs/2403.08245}
}

@misc{qiu2025demonsdetailimplementingload,
  title={Demons in the Detail: On Implementing Load Balancing Loss for Training Specialized Mixture-of-Expert Models},
  author={Qiu, Zihan and Huang, Zeyu and Zheng, Bo and Wen, Kaiyue and Wang, Zekun and Men, Rui and Titov, Ivan and Liu, Dayiheng and Zhou, Jingren and Lin, Junyang},
  year={2025},
  eprint={2501.11873},
  archivePrefix={arXiv},
  primaryClass={cs.LG},
  url={https://arxiv.org/abs/2501.11873}
}

@misc{qwen3,
  title={Qwen3 Technical Report},
  author={Qwen Team},
  year={2025},
  eprint={2505.09388},
  archivePrefix={arXiv},
  primaryClass={cs.CL},
  url={https://arxiv.org/abs/2505.09388}
}

@misc{longcat2025,
  title={LongCat-Flash Technical Report},
  author={Meituan AI Team},
  year={2025},
  eprint={2509.01322},
  archivePrefix={arXiv},
  primaryClass={cs.CL},
  url={https://arxiv.org/abs/2509.01322}
}

@misc{jin2025dtop_p,
  title={Sparsity-Controllable Dynamic Top-p MoE for Large Foundation Model Pre-training},
  author={Jin, Can and Peng, Hongwu and Xiang, Mingcan and Zhang, Qixin and Yuan, Xiangchi and Hasan, Amit and Dibua, Ohiremen and Gong, Yifan and Kang, Yan and Metaxas, Dimitris N.},
  year={2025},
  eprint={2512.13996},
  archivePrefix={arXiv},
  primaryClass={cs.AI},
  doi={10.48550/arXiv.2512.13996},
  url={https://arxiv.org/abs/2512.13996}
}

@inproceedings{zeng2024adamoe,
  title={{A}da{M}o{E}: Token-Adaptive Routing with Null Experts for Mixture-of-Experts Language Models},
  author={Zeng, Zihao and Miao, Yibo and Gao, Hongcheng and Zhang, Hao and Deng, Zhijie},
  booktitle={Findings of the Association for Computational Linguistics: EMNLP 2024},
  editor={Al-Onaizan, Yaser and Bansal, Mohit and Chen, Yun-Nung},
  month=nov,
  year={2024},
  address={Miami, Florida, USA},
  publisher={Association for Computational Linguistics},
  url={https://aclanthology.org/2024.findings-emnlp.361/},
  doi={10.18653/v1/2024.findings-emnlp.361},
  pages={6223--6235}
}

@inproceedings{yan2025tcmoe,
  title={{TC}-{M}o{E}: Augmenting Mixture of Experts with Ternary Expert Choice},
  author={Yan, Shen and Bin, Xingyan and Zhang, Sijun and Wang, Yisen and Lin, Zhouchen},
  booktitle={International Conference on Learning Representations (ICLR)},
  year={2025},
  url={https://openreview.net/forum?id=dsP91M4hDL},
  note={Poster}
}

@inproceedings{yang2024xmoe,
  title={{XM}o{E}: Sparse Models with Fine-grained and Adaptive Expert Selection},
  author={Yang, Yuanhang and Qi, Shiyi and Gu, Wenchao and Wang, Chaozheng and Gao, Cuiyun and Xu, Zenglin},
  booktitle={Findings of the Association for Computational Linguistics: ACL 2024},
  editor={Ku, Lun-Wei and Martins, Andre and Srikumar, Vivek},
  month=aug,
  year={2024},
  address={Bangkok, Thailand},
  publisher={Association for Computational Linguistics},
  url={https://aclanthology.org/2024.findings-acl.694/},
  doi={10.18653/v1/2024.findings-acl.694},
  pages={11664--11674}
}

@inproceedings{wang2025hmoe,
  title={{HM}o{E}: Heterogeneous Mixture of Experts for Language Modeling},
  author={Wang, An and Sun, Xingwu and Xie, Ruobing and Li, Shuaipeng and Zhu, Jiaqi and Yang, Zhen and Zhao, Pinxue and Han, Weidong and Kang, Zhanhui and Wang, Di and Okazaki, Naoaki and Xu, Cheng-zhong},
  booktitle={Proceedings of the 2025 Conference on Empirical Methods in Natural Language Processing},
  editor={Christodoulopoulos, Christos and Chakraborty, Tanmoy and Rose, Carolyn and Peng, Violet},
  month=nov,
  year={2025},
  address={Suzhou, China},
  publisher={Association for Computational Linguistics},
  url={https://aclanthology.org/2025.emnlp-main.1115/},
  doi={10.18653/v1/2025.emnlp-main.1115},
  pages={21943--21957},
  isbn={979-8-89176-332-6}
}

@article{liu2024routersvision,
  title={Routers in Vision Mixture of Experts: An Empirical Study},
  author={Liu, Tianlin and Blondel, Mathieu and Riquelme Ruiz, Carlos and Puigcerver, Joan},
  journal={Transactions on Machine Learning Research},
  year={2024},
  url={https://openreview.net/forum?id=aHk3vctnf1},
  note={Also available as arXiv:2401.15969}
}

@inproceedings{zhong2024lory,
  title={Lory: Fully Differentiable Mixture-of-Experts for Autoregressive Language Model Pre-training},
  author={Zhong, Zexuan and Xia, Mengzhou and Chen, Danqi and Lewis, Mike},
  booktitle={Conference on Language Modeling (COLM)},
  year={2024},
  url={https://openreview.net/forum?id=LKEJPySnlt}
}

@misc{wen2025seqtopk,
  title={Route Experts by Sequence, not by Token},
  author={Wen, Tiansheng and Wang, Yifei and Feng, Aosong and Ma, Long and Liu, Xinyang and Wang, Yifan and Guo, Lixuan and Chen, Bo and Jegelka, Stefanie and You, Chenyu},
  year={2025},
  eprint={2511.06494},
  archivePrefix={arXiv},
  primaryClass={cs.LG},
  doi={10.48550/arXiv.2511.06494},
  url={https://arxiv.org/abs/2511.06494}
}

@misc{zhang2019rmsnorm,
  title={Root Mean Square Layer Normalization},
  author={Zhang, Biao and Sennrich, Rico},
  year={2019},
  eprint={1910.07467},
  archivePrefix={arXiv},
  primaryClass={cs.LG},
  url={https://arxiv.org/abs/1910.07467}
}

@article{su2024rope,
  title={RoFormer: Enhanced Transformer with Rotary Position Embedding},
  author={Su, Jianlin and Ahmed, Murtadha and Lu, Yu and Pan, Shengfeng and Bo, Wen and Liu, Yunfeng},
  year={2024},
  journal={Neurocomputing},
  volume={568},
  pages={127063},
  doi={10.1016/j.neucom.2023.127063},
  url={https://www.sciencedirect.com/science/article/pii/S0925231223011864}
}

@misc{ainslie2023gqa,
  title={GQA: Training Generalized Multi-Query Transformer Models from Multi-Head Checkpoints},
  author={Ainslie, Joshua and Lee-Thorp, James and de Jong, Michiel and Zemlyanskiy, Yury and Lebron, Federico and Sanghai, Sumit},
  year={2023},
  eprint={2305.13245},
  archivePrefix={arXiv},
  primaryClass={cs.CL},
  url={https://arxiv.org/abs/2305.13245}
}

@misc{jordan2024muon,
  title={Muon: An optimizer for hidden layers in neural networks},
  author={Jordan, Keller},
  year={2024},
  howpublished={\url{https://kellerjordan.github.io/posts/muon/}},
  note={Blog post}
}

@inproceedings{loshchilov2019adamw,
  title={Decoupled Weight Decay Regularization},
  author={Loshchilov, Ilya and Hutter, Frank},
  booktitle={International Conference on Learning Representations},
  year={2019},
  url={https://openreview.net/forum?id=Bkg6RiCqY7}
}

@misc{yang2022mup,
  title={Tensor Programs V: Tuning Large Neural Networks via Zero-Shot Hyperparameter Transfer},
  author={Yang, Greg and Hu, Edward J. and Babuschkin, Igor and Sidor, Szymon and Liu, Xiaodong and Farhi, David and Ryder, Nick and Pachocki, Jakub and Chen, Weizhu and Gao, Jianfeng},
  year={2022},
  eprint={2203.03466},
  archivePrefix={arXiv},
  primaryClass={cs.LG},
  url={https://arxiv.org/abs/2203.03466}
}

@misc{team2024gemma2,
  title={Gemma 2: Improving Open Language Models at a Practical Size},
  author={Gemma Team and Riviere, Morgane and Pathak, Shreya and Sessa, Pier Giuseppe and Cassirer, Cassidy and Coppey, Lucie and El-Boukkouri, Khalid and others},
  year={2024},
  eprint={2408.00118},
  archivePrefix={arXiv},
  primaryClass={cs.CL},
  url={https://arxiv.org/abs/2408.00118}
}

@misc{so2021primer,
  title={Primer: Searching for Efficient Transformers for Language Modeling},
  author={So, David R. and Mańke, Wojciech and Liu, Hanxiao and Dai, Zihang and Shazeer, Noam and Le, Quoc V.},
  year={2021},
  eprint={2109.08668},
  archivePrefix={arXiv},
  primaryClass={cs.LG},
  url={https://arxiv.org/abs/2109.08668}
}

@misc{dehghani2023vit22b,
  title={Scaling Vision Transformers to 22 Billion Parameters},
  author={Dehghani, Mostafa and Djolonga, Josip and Mustafa, Basil and Padlewski, Piotr and Heek, Jonathan and Gilmer, Justin and Steiner, Andreas and Caron, Mathilde and Geirhos, Robert and Alabdulmohsin, Ibrahim and others},
  year={2023},
  eprint={2302.05442},
  archivePrefix={arXiv},
  primaryClass={cs.CV},
  url={https://arxiv.org/abs/2302.05442}
}

@misc{liu2025moonlight,
  title={Moonlight: A Cost-Effective Approach for Pre-training Large Language Models},
  author={Liu, Haokun and Li, Yifan and Shen, Yingxia and Wang, Bo and Liang, Chen and Jiang, Chengyue and Li, Chonghua and Deng, Damai and Ding, Fangxiaoyu and Gao, Wei and others},
  year={2025},
  eprint={2502.16456},
  archivePrefix={arXiv},
  primaryClass={cs.CL},
  url={https://arxiv.org/abs/2502.16456}
}

@inproceedings{ioffe2015batch,
  title={Batch Normalization: Accelerating Deep Network Training by Reducing Internal Covariate Shift},
  author={Ioffe, Sergey and Szegedy, Christian},
  booktitle={International Conference on Machine Learning},
  pages={448--456},
  year={2015},
  organization={PMLR}
}

@inproceedings{radford2021learning,
  title={Learning Transferable Visual Models From Natural Language Supervision},
  author={Radford, Alec and Kim, Jong Wook and Hallacy, Chris and Ramesh, Aditya and Goh, Gabriel and Agarwal, Sandhini and Sastry, Girish and Askell, Amanda and Mishkin, Pamela and Clark, Jack and others},
  booktitle={International Conference on Machine Learning},
  pages={8748--8763},
  year={2021},
  organization={PMLR}
}

@inproceedings{he2020momentum,
  title={Momentum Contrast for Unsupervised Visual Representation Learning},
  author={He, Kaiming and Fan, Haoqi and Wu, Yuxin and Xie, Saining and Girshick, Ross},
  booktitle={IEEE/CVF Conference on Computer Vision and Pattern Recognition},
  pages={9729--9738},
  year={2020}
}

@inproceedings{caron2021emerging,
  title={Emerging Properties in Self-Supervised Vision Transformers},
  author={Caron, Mathilde and Touvron, Hugo and Misra, Ishan and J{\'e}gou, Herv{\'e} and Mairal, Julien and Bojanowski, Piotr and Joulin, Armand},
  booktitle={IEEE/CVF International Conference on Computer Vision},
  pages={9650--9660},
  year={2021}
}

@inproceedings{kingma2015adam,
  title={Adam: A Method for Stochastic Optimization},
  author={Kingma, Diederik P and Ba, Jimmy},
  booktitle={International Conference on Learning Representations},
  year={2015}
}
\bibliographystyle{icml2026}

\appendix

\section{Future Information Leakage for Expert Choice Models}
\label{app:future_leakage_ec}

In DeepSeek's loss-free load balancing paper~\citep{wang2024auxiliarylossfreeloadbalancingstrategy}, they give an upper bound on the future information leakage of Expert Choice (EC) to be superlinear in the number of tokens. They considered all potential selection combinations $\binom{N}{k}$ for choose $k=N/E$ tokens out of $N$ tokens, which makes upper bound $\log_2\binom{N}{k} = O(N \log N)$.

We consider two scenarios: when cutoff threshold is expressed as a finite precision float, it is trivial that the total future information leakage is at most the number of bits to represent the cutoff threshold. However, when cutoff threshold is of infinite precision, we show that we can indeed leak at least $O(N \log N)$ bits of future information, making the bound tight.

We arrange the following sections as follows:

\begin{enumerate}
    \item First, we provide a formal definition of future information leakage.
    \item Then, we show for finite precision, the total leakage is constant, which means per token leakage is 0 as batch size increases.
    \item Then, we show for infinite precision, we describe an encoding strategy that can leak at least $O(N \log N)$ bits of future information, making the bound tight. The idea is that we can break the cutoff space into $2^N$ small intervals, and injectively (though not surjectively) map each potential selection combination to a unique interval in a memoryless way. 
    \item We formally prove that the encoding strategy can leak at least $O(N \log N)$ bits of future information.
\end{enumerate}

Gladly, since we rely on finite precision float to represent the cutoff threshold, ET is still causal. 

\subsection{Definition of Future Information Leakage}

\begin{definition}[Future information leakage]
Fix a sequence length $N$ and a deterministic selection rule $F$ that maps a logit sequence $r_{1:N}$ to a subset $F(r_{1:N}) \subseteq [N]$ (where $[N] \triangleq \{1,\dots,N\}$). Here $r_{1:N}$ denotes the entire sequence of router scores/logits across tokens, and we write $z_t \triangleq \mathbf{1}\!\left[t\in F(r_{1:N})\right]$ for the induced selection indicator.

An \emph{advice variable} is any function $A=\alpha(r_{1:N})$ with finite range, where $\alpha$ is an \emph{encoder} that maps the full logit sequence to a finite label. We write $\mathrm{Range}(\alpha)\triangleq \{\alpha(r_{1:N}) : r_{1:N}\}$ for the set of labels that can be produced by $\alpha$.

We say $A$ \emph{causalizes} $F$ if there exist functions $\{g_t\}_{t=1}^N$ such that for all $r_{1:N}$ and all $t\in[N]$,
\[
z_t \;=\; g_t\!\left(r_{1:t}, A\right).
\]
The future information leakage of $F$ on length $N$ is
\[
\mathcal{L}_F([1\!:\!N]) \;\triangleq\; \min_{\alpha,\,\{g_t\}} \log_2\!\left|\mathrm{Range}(\alpha)\right|.
\]
\end{definition}

\subsection{Finite-precision cutoff implies constant leakage}

With this definition, it is trivial to show that, under finite precision float representation, the total future information leakage is constant, which means per token leakage is 0 as batch size increases.

In our case, $F$ is the Expert Choice selection rule induced by a cutoff threshold: for each expert, tokens are selected by comparing their router scores against the expert's cutoff (equivalently, selecting those above the cutoff, which matches top-$k$ when the cutoff is set to the $k$-th order statistic).

\begin{theorem}[Finite-precision cutoff implies constant leakage]
If the cutoff threshold is represented with $b$ bits of precision (e.g., $b=16$ for \texttt{bf16} or $b=32$ for \texttt{fp32}), then the future information leakage satisfies $\mathcal{L}_F([1\!:\!N]) \le b$ for all $N$.
\end{theorem}

\begin{proof}
This is an upper bound via a particular advice choice. Let the advice be the cutoff itself: $A=\alpha(r_{1:N}) \triangleq \beta$, encoded in $b$ bits, so $|\mathrm{Range}(\alpha)| \le 2^b$. Given $A=\beta$, the selection indicator at time $t$ is a causal function of the prefix (in fact, of $r_t$) and $\beta$ by thresholding, hence $A$ causalizes $F$. Therefore $\mathcal{L}_F([1\!:\!N]) \le \log_2 |\mathrm{Range}(\alpha)| \le b$.
\end{proof}

\subsection{Upper bound on Infinite-precision future information leakage}

The strategy we used in the previous subsection does not work for infinite precision cutoff, as it takes infinite bits of communication to represent the cutoff. Thus, we need a different encoding strategy.

Previously, Loss-free Load Balancing paper~\citet{wang2024auxiliarylossfreeloadbalancingstrategy} gave a combinatorial upper bound on the information carried by an Expert Choice allocation when we allow all admissible token-to-expert assignments consistent with the sparsity pattern.

Using the token-choice notation from Loss-free Load Balancing, let $N$ denote the number of tokens in the routing pool and $GE$ the number of routed experts. Each token activates $G$ routed experts, so the MoE sparsity is $\frac{1}{E}$. For an MoE layer in Expert Choice, the maximum information leakage $L$ (bits per token) is:
\begin{align*}
L
&= \frac{GE}{N}\log_2\!\binom{N}{\frac{N}{E}} \nonumber\\
&> \frac{GE}{N}\cdot\frac{N}{E}\log_2(E-1) \nonumber\\
&= G\log_2(E-1).
\tag{7}
\end{align*}

For a model with sparsity $\frac{1}{E}=\frac{2}{16}=0.125$ and 9 MoE layers, the total leakage information is more than 50 bits per token.

\subsection{Encoding Strategy and Decoding Procedure}

The above upper bound is pretty intuitive to understand. We ask, can we reach this bound? Surprisingly, we can. In this subsection, we describe an encoding strategy where we create an injective mapping from the set of all possible expert choice combinations to the range of the infinite precision cutoff.

Suppose the cutoff $c \in [0,1]$. We partition this space into $2^N$ dyadic intervals. Let $z \in \{0,1\}^N$ be the target expert choice combination (codeword). We map $z$ to the interval:
\[
I(z)\;\triangleq\;\Big[\sum_{t=1}^N (1-z_t)2^{-t},\ \sum_{t=1}^N (1-z_t)2^{-t} + 2^{-N}\Big).
\]
This orders codewords in descending order: $00\cdots 00$ maps to the rightmost interval and $11\cdots 11$ maps to the leftmost. Choose a cutoff value $c\in I(z)$. Obviously, this mapping is injective, i.e. each codeword maps to a unique interval. However, since not all combinations are possible, it is not surjective.

\paragraph{Information from the Past.}
For a token $t$, what information about the past routing logits $r_{1:t}$ and decisions $z_{1:t}$ can we use to guide the selection of the interval? Turns out, the only information we get is the current bracket $[\ell, u)$ containing the decision boundary. The upper bound $u$ is determined by the lowest logit of the selected tokens, and the lower bound $\ell$ by the highest logit of the unselected tokens. To maximize information leakage, we minimize help from the past by always choosing the next query logit to be in the middle of the interval, $r_t = (\ell+u)/2$. Then, the routing decision $z_t = \mathbf{1}\{r_t \ge c\}$ reveals exactly one bit of the cutoff, requiring full future information.

\paragraph{Decoding Procedure.}
We formalize this procedure in Algorithm~\ref{alg:cutoff_decode}.

\begin{algorithm}[t]
\caption{Binary-search decoding of an infinite-precision cutoff}
\label{alg:cutoff_decode}
\begin{algorithmic}[1]
\STATE \textbf{Input:} horizon $N$; unknown cutoff $c\in[0,1]$; oracle bit $z_t=\mathbf{1}\{r_t\ge c\}$
\STATE $\ell\gets 0,\ u\gets 1$
\FOR{$t=1,\ldots,N$}
\STATE $r_t\gets (\ell+u)/2$ \COMMENT{Query midpoint}
\STATE Observe selection $z_t \in \{0,1\}$
\IF{$z_t=1$}
\STATE $u\gets r_t$ \COMMENT{Selected ($r_t \ge c$) $\implies c \in [\ell, r_t)$}
\ELSE
\STATE $\ell\gets r_t$ \COMMENT{Not selected ($r_t < c$) $\implies c \in [r_t, u)$}
\ENDIF
\ENDFOR
\STATE \textbf{Return} $(z_1,\ldots,z_N)$ and interval $[\ell,u)$
\end{algorithmic}
\end{algorithm}

\begin{figure}[t]
\centering
\includegraphics[width=0.9\linewidth]{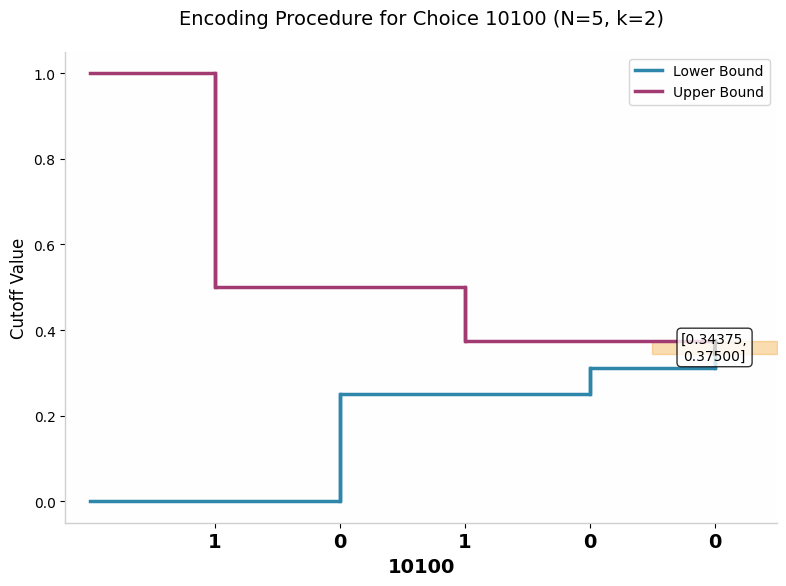}
\caption{Illustration of the binary-search encoding strategy. The cutoff value $c$ partitions the interval $[0,1]$ into regions corresponding to different expert selection patterns, enabling $N$ bits of future information to be encoded in a single real-valued threshold.}
\label{fig:encoding}
\end{figure}

\subsection{Formal Proof}

We now formally prove that the amount of information required to implement the Expert Choice selection causally is lower-bounded by the combinatorial entropy of the selection space. This confirms that the upper bound in Eq.~(7) is tight and that the infinite-precision construction in Algorithm~\ref{alg:cutoff_decode} is optimal in terms of information leakage.

\begin{theorem}
\label{thm:lower_bound}
For a single expert with capacity $k = \lfloor N/E \rfloor$, any causal routing mechanism that can realize all possible top-$k$ assignments requires at least $\log_2 \binom{N}{k}$ bits of non-causal information (advice).
\end{theorem}

\begin{proof}
Let $\mathcal{Z}_k = \{z \in \{0,1\}^N : \sum_{t=1}^N z_t = k\}$ be the set of all valid selection indicators for the expert, with $|\mathcal{Z}_k| = \binom{N}{k}$. We show that distinct advice is necessary for every distinct pattern in $\mathcal{Z}_k$.

Consider the specific family of logit sequences generated by the binary search process in Algorithm~\ref{alg:cutoff_decode}. In this construction, the router logit $r_t$ is the midpoint of the current valid interval $[\ell, u)$, which is determined solely by the history of decisions $z_1, \dots, z_{t-1}$. Consequently, if two selection patterns $z$ and $z'$ share the same prefix $z_{1:t-1}$, they will generate the exact same logit $r_t$ at step $t$.

Suppose there exists an advice encoding with range size strictly less than $\binom{N}{k}$. By the Pigeonhole Principle, at least two distinct valid patterns $z, z' \in \mathcal{Z}_k$ must share the same advice value. Let $t$ be the first index where they differ (i.e., $z_{1:t-1} = z'_{1:t-1}$ but $z_t \neq z'_t$).
Since the prefixes are identical, the generated logit sequence $r_{1:t}$ is identical for both patterns. A causal decoder, which must output a decision at time $t$ based only on $r_{1:t}$ and the advice, receives identical inputs for both cases. It must therefore produce the same output. However, $z_t \neq z'_t$, so the decoder necessarily fails for at least one of the patterns.

Thus, every valid top-$k$ pattern requires a unique advice value. The minimum information leakage is $\log_2 \binom{N}{k}$ bits per expert.
\end{proof}

Summing this lower bound over $GE$ experts recovers the combinatorial quantity in Eq.~(7), proving that Expert Choice routing fundamentally requires significant future information to implement.

\section{Architecture Details}
\label{app:arch_diff}

Our model architecture follows nanochat~\cite{karpathy2025nanochat}, which differs from standard GPT-2~\cite{radford2019language} in several ways. Table~\ref{tab:arch_config} summarizes the key differences.

\begin{table*}[t]
\centering
\caption{Model architecture and size configurations. Architecture features are shared between d12 and d20 (nanochat-style). For MoE variants with $G{=}1$, $E{=}16$: 16 routed experts + 1 shared = 17 total experts. Total params include all expert parameters; active params include only the shared expert plus on average one routed expert per token.}
\label{tab:arch_config}
\small
\begin{tabular}{@{}l|>{\centering\arraybackslash}p{3.5cm}|>{\centering\arraybackslash}p{3cm}|>{\centering\arraybackslash}p{3cm}@{}}
\toprule
\textbf{Feature} & \textbf{GPT-2} & \textbf{d12} & \textbf{d20} \\
\midrule
Tokenization & GPT-2 & \multicolumn{2}{c|}{RustBPE} \\
Vocab Size & 50,257 & \multicolumn{2}{c|}{65,536} \\

Activation & GELU & \multicolumn{2}{c|}{ReLU$^2$~\cite{so2021primer}} \\
Normalization & LayerNorm & \multicolumn{2}{c|}{RMSNorm~\cite{zhang2019rmsnorm}} \\
FFN Dimension & $4 \times d_{\text{model}}$ & \multicolumn{2}{c|}{$4 \times d_{\text{model}}$} \\
Linear Layer Bias & Yes & \multicolumn{2}{c|}{No} \\
Position Encoding & Learned & \multicolumn{2}{c|}{RoPE~\cite{su2024rope}} \\
Head Dimension & 64 & \multicolumn{2}{c|}{128} \\
QK Normalization & No & \multicolumn{2}{c|}{Yes~\cite{dehghani2023vit22b}} \\
Logits Softcapping & No & \multicolumn{2}{c|}{15.0~\cite{team2024gemma2}} \\
Embedding Weights & Tied (wte = lm\_head) & \multicolumn{2}{c|}{Untied} \\

First Layer Dense & --- & \multicolumn{2}{c|}{Yes} \\
\midrule
$n_{\text{embd}}$ & 768 & 768 & 1280 \\
$n_{\text{layer}}$ & 12 & 12 & 20 \\
$n_{\text{head}}$ & 12 & 6 & 10 \\
KV heads & 12 & 2 & 2 \\
Dense Params & 124M & 195M & 561M \\
MoE Total Params & --- & 575M & 2429M \\
MoE Active Params & --- & 195M & 561M \\
\bottomrule
\end{tabular}
\end{table*}

\subsection{Data and Tokenization}

We train on the FineWeb-Edu 100B shuffle dataset~\cite{penedo2024the}, a high-quality educational web corpus. Tokenization uses RustBPE with a vocabulary of 65,536 tokens (64k, power-of-2 aligned for GPU efficiency). We use a sequence length of 2048 tokens.

\subsection{Model Size Configurations}
\label{app:model_sizes}

Table~\ref{tab:arch_config} shows the model configurations used in our experiments. We follow the nanochat naming convention where \texttt{d*} indicates the number of layers, and $n_{\text{embd}} = d \times 64$, with head dimension fixed at 128.

\paragraph{Attention configuration.} We use grouped-query attention~\cite{ainslie2023gqa} with head dimension 128 (larger than GPT-2's 64). The number of attention heads is $n_{\text{head}} = n_{\text{embd}} / 128$, giving 6 heads for d12 and 10 heads for d20. QK normalization is applied before the attention computation.

\paragraph{MoE configuration.} For MoE variants, we use 16 routed experts with granularity $G{=}1$ and expansion $E{=}16$, plus 1 shared expert (17 total). Each routed and shared expert has dimension $d_{\text{expert}} = 2 \times n_{\text{embd}}$ (half the dense FFN dimension). The shared expert processes every token, while each token activates on average 1 routed expert, matching the dense model's active parameter count.

\section{Training Setup Details}
\label{app:training_setup}

\subsection{Training Hyperparameters}

\begin{table}[h]
\centering
\caption{Training hyperparameters.}
\label{tab:training_hyperparams}
\small
\begin{tabular}{@{}ll@{}}
\toprule
\textbf{Hyperparameter} & \textbf{Value} \\
\midrule
Total tokens & 10B / 11.2B \\
Batch size (tokens) & 524,288 (0.5M) \\
Sequence length & 2048 \\
Muon Warmup steps & No warmup\\
AdamW Warmup steps & 250 \\
Learning rate schedule & Linear decay \\
Min learning rate & $0.1 \times$ peak LR \\
Gradient clipping & None \\
Weight decay & 0.0 \\
AdamW $\beta_1, \beta_2$ & 0.9, 0.95 \\
\bottomrule
\end{tabular}
\end{table}

\subsubsection{Weight Initialization}

We follow the nanochat initialization scheme~\cite{karpathy2025nanochat}, which uses aspect-ratio scaled initialization. For a weight matrix $W \in \mathbb{R}^{d_{\text{out}} \times d_{\text{in}}}$:

\begin{equation}
\text{std} = \frac{1}{\sqrt{d_{\text{in}}}} \cdot \min\left(1, \sqrt{\frac{d_{\text{out}}}{d_{\text{in}}}}\right)
\end{equation}

This formula reduces to standard $1/\sqrt{d_{\text{in}}}$ initialization for square or tall matrices, but scales down variance for wide matrices where $d_{\text{out}} \gg d_{\text{in}}$.

\paragraph{Component-specific initialization:}
\begin{itemize}
    \item \textbf{Embeddings}: $\mathcal{N}(0, 1)$ -- standard normal initialization
    \item \textbf{Output projections} (\texttt{lm\_head}, \texttt{c\_proj}): Zero initialization, critical for Muon optimizer stability
    \item \textbf{Router weights}: $\mathcal{N}(0, 1/\sqrt{d_{\text{in}}})$ -- small init for symmetry breaking
    \item \textbf{Expert weights}: Aspect-ratio scaled for up projections, zero for down projections
    \item \textbf{Attention weights}: Aspect-ratio scaled as above
\end{itemize}

\begin{table}[h]
\centering
\caption{Parameter initialization and optimizer configuration. Aspect-ratio scaled init.\ uses $\text{std} = d_{\text{in}}^{-1/2} \cdot \min(1, \sqrt{d_{\text{out}}/d_{\text{in}}})$. LRs include $\mu$P~\cite{yang2022mup} scaling $\lambda = (d_{\text{model}}/768)^{-1/2}$.}
\label{tab:param_init_lr}
\footnotesize
\begin{tabular}{@{}lccc@{}}
\toprule
\textbf{Parameter} & \textbf{Init.} & \textbf{LR} & \textbf{Opt.} \\
\midrule
$W_E$ (embed.) & $\mathcal{N}(0, 1)$ & $0.2\lambda$ & AdamW \\
$W_{\text{lm}}$ (head) & $\mathbf{0}$ & $0.004\lambda$ & AdamW \\
\midrule
$W_{QKV}$ (attn) & Asp.-ratio & $0.02\lambda$ & Muon \\
$W_O$ (attn proj) & $\mathbf{0}$ & $0.02\lambda$ & Muon \\
\midrule
$W_{\text{router}}$ & $\mathcal{N}(0, d^{-\frac{1}{2}})$ & $0.02\lambda$ & Muon \\
$W^{(e)}_{\uparrow}$ (exp.\ up) & Asp.-ratio & $0.02\lambda$ & Muon \\
$W^{(e)}_{\downarrow}$ (exp.\ dn) & $\mathbf{0}$ & $0.02\lambda$ & Muon \\
$W^{(s)}_{\uparrow}$ (shd.\ up) & Asp.-ratio & $0.02\lambda$ & Muon \\
$W^{(s)}_{\downarrow}$ (shd.\ dn) & $\mathbf{0}$ & $0.02\lambda$ & Muon \\
\bottomrule
\end{tabular}
\end{table}

\subsubsection{Optimizer Configuration}

We use a hybrid optimizer setup following nanochat~\cite{jordan2024muon, liu2025moonlight}:
\begin{itemize}
    \item \textbf{Muon}~\cite{jordan2024muon} for 2D/3D weight matrices (attention, MLP, experts): momentum-based optimizer with Newton-Schulz orthogonalization
    \item \textbf{AdamW}~\cite{loshchilov2019adamw} for embeddings and output head: with learning rate scaling $\propto 1/\sqrt{d_{\text{model}}/768}$
\end{itemize}

No weight decay is used, as Muon provides implicit regularization and language models benefit from memorization.

\subsection{Hardware Infrastructure Details}

\paragraph{Hardware.} We train our models on a single node with 8x NVIDIA B200 GPUs, each with 180GB of memory. 

\paragraph{Code.} For TC models, we use the ScatterMoE backend \citep{tan2024scattermoe}. For EC and ET models, we write our own custom Pytorch MoE implementation. We use padding to handle variable number of tokens per expert. 

\paragraph{Parallelism.} We rely on Nanochat's implementation of distributed AdamW and Muon optimizer use a ZeRO-2 style gradient synchronization \citep{rajbhandari2020zero}. For EC and ET models, we write our own expert parallelization all-to-all communication framework. This allows us to use maximum batch size during routing instead of micro-batches, reaching a better usage/cutoff variance trade-off while saving memory.

\section{CORE Evaluation Details}
\label{app:coreeval}

We evaluate using the CORE benchmark~\cite{li2024dclm}, which provides a standardized suite of in-context learning tasks for language model evaluation.

\paragraph{Task types.} The CORE benchmark includes multiple-choice tasks, schema matching tasks, and language modeling tasks, testing various aspects of language understanding.

\paragraph{Metric.} The primary metric is \emph{centered accuracy}, which adjusts for random baseline performance:
\begin{equation}
\text{acc}_{\text{centered}} = \frac{\text{acc} - 0.01 \times \text{baseline}_{\text{random}}}{1.0 - 0.01 \times \text{baseline}_{\text{random}}}
\end{equation}
This normalization ensures that random guessing yields a score near zero, while perfect accuracy yields 1.0. The final CORE Eval score is the mean of centered accuracy across all tasks.

\paragraph{Evaluation protocol.} We evaluate at fixed intervals during training (every 250 steps by default) to track learning dynamics.

\section{Ablations}
\label{app:ablations}

\subsection{Warmup}

We find warmup crucial for ET. In the early stages of training, the cutoff threshold is not yet stable, while the EMA lags behind the actual cutoff threshold because of slow update speed ($1/(1-\beta) \approx 1000$ steps). As a result, the threshold-based routing becomes unreliable: tokens that should be routed are dropped, and the capacity lower bound is frequently triggered (Figure~\ref{fig:warmup}c). This leads to undertrained experts during early training. To address this, we warm up the routing by using TopK selection for the first 4,000 steps before switching to threshold-based routing. ET no warmup relies solely on the capacity factor during these early steps, which is suboptimal because capacity control can limit collapse but does not provide a stable threshold estimate or balanced expert learning signal before the cutoff EMA has converged. As shown in Figure~\ref{fig:warmup}, warmup stabilizes the cutoff-EMA trajectory (a, d), increases raw expert usage (b), and reduces starvation rate (c). We also observe that ET no warmup exhibits higher variance in both logits (e) and gate outputs (f), suggesting less stable gradient signals.

\begin{figure*}[t]
    \centering
    \begin{tabular}{@{}ccc@{}}
        \includegraphics[width=0.32\textwidth]{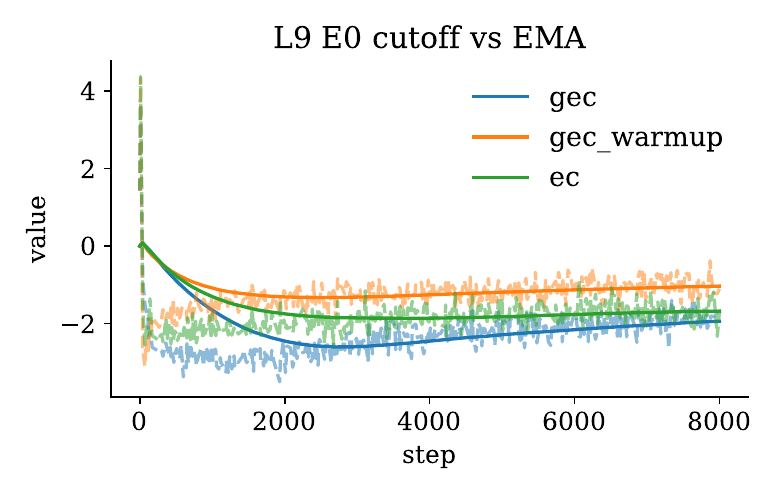} &
        \includegraphics[width=0.32\textwidth]{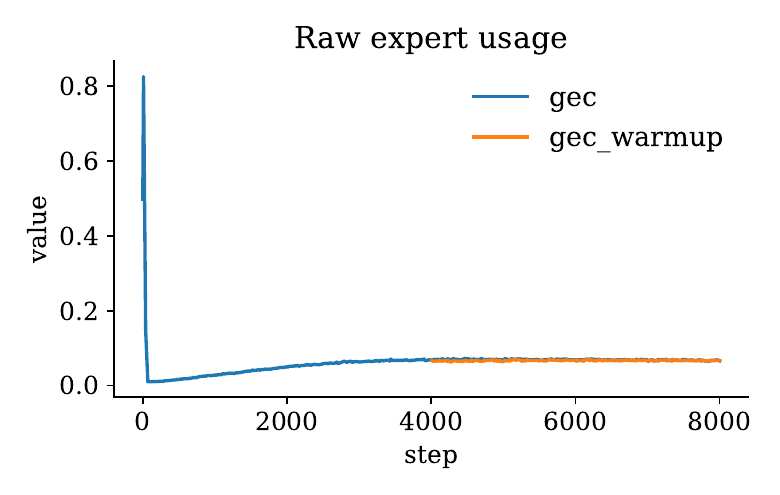} &
        \includegraphics[width=0.32\textwidth]{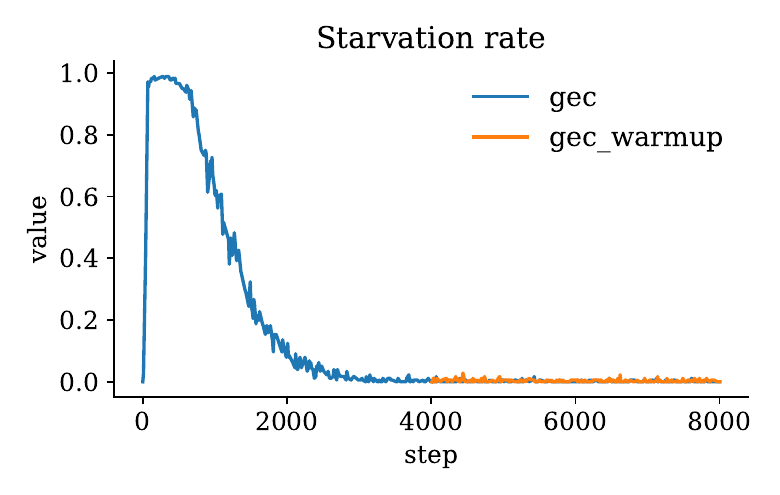} \\
        (a) L9 cutoff vs EMA & (b) Raw expert usage & (c) Starvation rate \\[6pt]
        \includegraphics[width=0.32\textwidth]{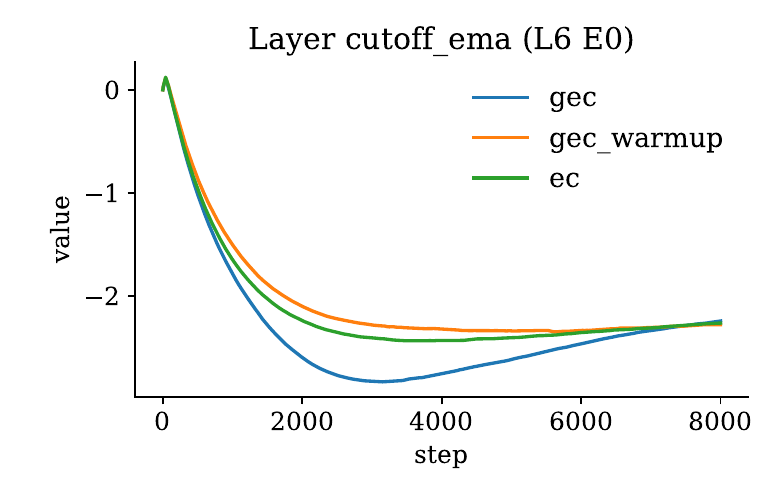} &
        \includegraphics[width=0.32\textwidth]{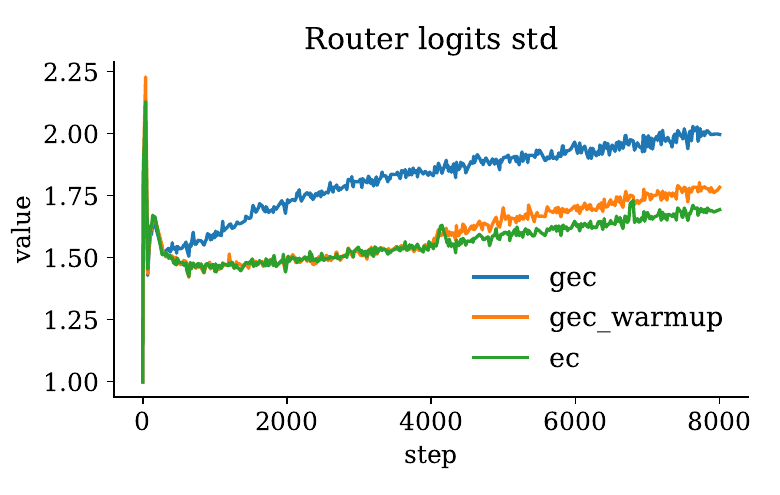} &
        \includegraphics[width=0.32\textwidth]{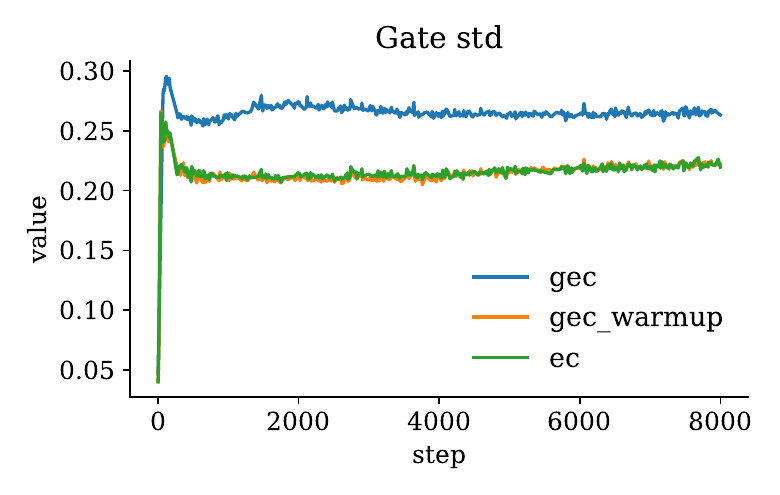} \\
        (d) L6 cutoff-EMA & (e) Router logits std & (f) Gate std \\
    \end{tabular}
    \caption{Effect of TopK warmup on ET training dynamics (first 8k steps). Before 4k steps, ET no warmup exhibits unstable threshold routing: (a) the cutoff-EMA lags behind the actual cutoff, (b) raw expert usage is low, and (c) starvation rate is high as the capacity lower bound is frequently triggered. ET no warmup relies only on the capacity factor during this stage, which is suboptimal because it does not provide a stable threshold estimate or balanced expert learning signal. With warmup, the cutoff-EMA trajectory stabilizes (d), and router outputs show lower variance in both logits (e) and gates (f). Note: \texttt{ec\_shared\_bsz512k} does not log raw usage and underflow metrics, so panels (b) and (c) show only the two ET runs.}
    \label{fig:warmup}
\end{figure*}

\subsection{Comparison to Token Choice}

In our setup, Token Choice with loss-free load balancing shows a less stable routing trajectory than ET and EC, especially in early layers. Figure~\ref{fig:l1_cutoff_ema_deepseek_vs_ec} compares cutoff-EMA (expert 0) at layer 1. ET and EC stabilize quickly, while DeepSeek's loss-free controller drifts upward over training. We treat this as an exploratory observation rather than a central claim, since the behavior may depend on hyperparameters and gating parameterization.

\begin{figure}[h]
    \centering
    \includegraphics[width=0.9\linewidth]{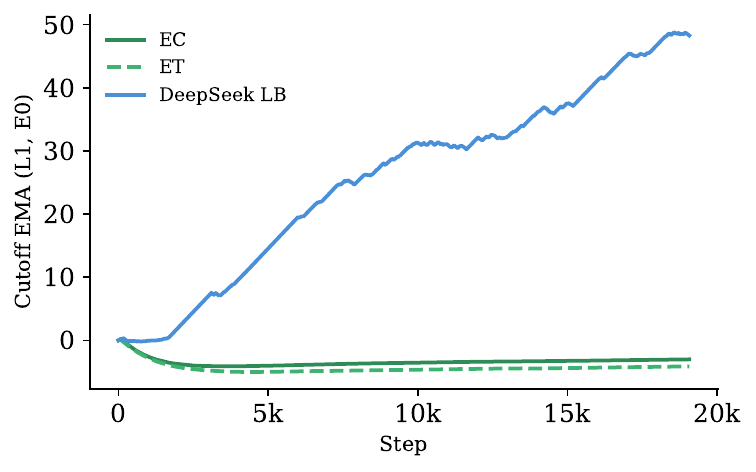}
    \caption{Layer-1 cutoff-EMA (expert 0) under EC/ET vs DeepSeek loss-free load balancing.}
    \label{fig:l1_cutoff_ema_deepseek_vs_ec}
\end{figure}

\subsection{Shared Expert}

We report results of EC and ET no warmup with and without shared expert. For no shared experts, we select 2 experts out of 16 routed experts, roughly matching the number of parameters and compute to the shared variant. In both cases, shared expert improves loss by roughly 0.02. We suspect that while later layers need early layers to empower the router, sometimes early layers have no activated experts, causing ineffective routing. See Table~\ref{tab:shared_expert_ablation} for more details.

\begin{table}[h]
\centering
\caption{Ablation on the shared expert mechanism. In both ET no warmup and EC, shared expert improves loss by roughly 0.02. }
\label{tab:shared_expert_ablation}
\small
\setlength{\tabcolsep}{4pt}
\begin{tabularx}{\columnwidth}{@{}>{\raggedright\arraybackslash}X c r r@{}}
\toprule
\textbf{Method} & \textbf{Shared} & \textbf{CE} & \textbf{CORE} \\
\midrule
EC (bsz 512k) & Yes & 2.843 & 19.94 \\
EC (bsz 512k) & No  & 2.862 & 16.307\\
\midrule
ET no warmup ($\beta{=}0.999$) & Yes & 2.844 & 16.867 \\
ET no warmup ($\beta{=}0.999$) & No  & 2.862 & 18.515 \\
\bottomrule
\end{tabularx}
\end{table}

\subsection{Normalization}

Initially, we assumed that dynamic expert count would bring instability in training because of the scale expansion. However, we found that normalization was ineffective in our setting. No norm outperformed fanout norm by 0.04 in CE loss. We suspect that the norm made experts' contribution unpredictable (see Figure~\ref{fig:norm_comparison}).

\begin{figure}[h]
    \centering
    \includegraphics[width=0.8\linewidth]{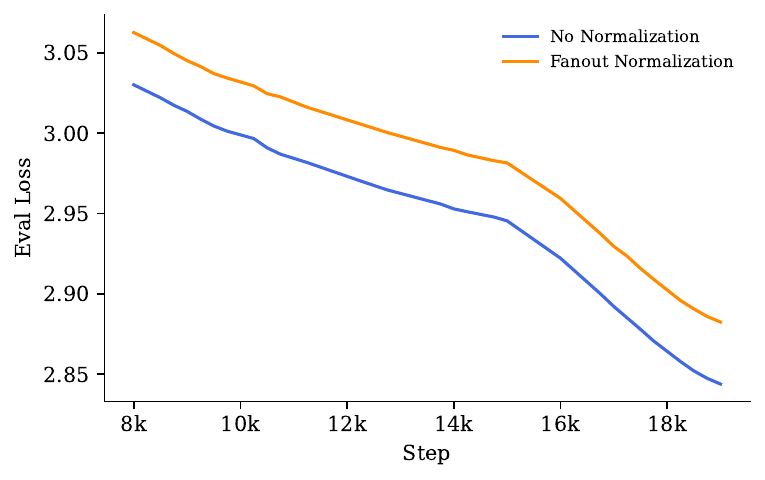}
    \caption{Comparison of evaluation loss with and without normalization. The configuration without normalization (blue) consistently achieves lower loss than the fanout-normalized variant (orange).}
    \label{fig:norm_comparison}
\end{figure} 

\section{Additional Experiment Results}
\label{app:additional_results}

\subsection{Capacity Constraints}
\label{app:capacity}

Because ET's thresholding does not fix the per-batch number of selected tokens for each expert, expert loads can fluctuate around the target, which can risk GPU out-of-memory. Following standard practice~\cite{fedus2022switch}, we enforce capacity constraints during training: each expert processes between $(1-C) \cdot N / E$ and $(1+C) \cdot N / E$ tokens per batch (capacity factor $C = 0.5$), with excess tokens dropped or capacity padded. Since these constraints are absent at inference, frequent triggering would cause train-inference mismatch. Figure~\ref{fig:capacity_constraints} shows that \textbf{capacity constraints are triggered infrequently}: after warmup, both saturation and starvation rates remain low. This confirms that train-inference mismatch from capacity constraints is minimal.

Figure~\ref{fig:capacity_constraints} shows capacity constraint metrics for ET (with warmup) from step 4k onward. After warmup, raw expert usage stabilizes around 6.5\%, and both saturation and starvation rates remain low, confirming that capacity constraints are rarely triggered and train-inference mismatch is minimal.

\begin{figure*}[t]
    \centering
    \includegraphics[width=0.95\textwidth]{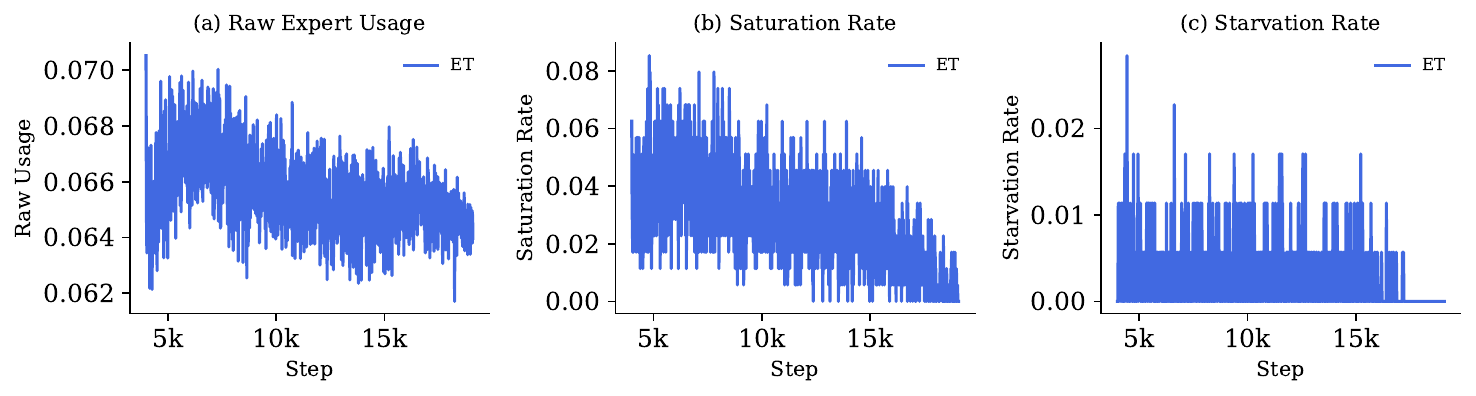}
    \caption{Capacity constraint behavior during ET training (from step 4k onward, after warmup). (a) Raw expert usage before capacity capping. (b) Saturation rate: fraction of selected tokens dropped due to capacity limits. (c) Starvation rate: fraction of unused expert capacity. Both saturation and starvation rates remain low, confirming minimal train-inference mismatch.}
    \label{fig:capacity_constraints}
\end{figure*}

\subsection{Routing Consistency Sweep}
\label{app:routing_consistency}

Section~\ref{sec:routing_consistency} reports the main weighted Jaccard heatmap strip. Here we define the routing-consistency metrics used throughout the comparison and include the companion joint JSD heatmap for the same four runs.

\paragraph{Objects being compared.}
For a given token-layer pair, let $A_t$ and $B_t$ denote the sets of active routed experts under checkpoints $A$ and $B$. The shared expert is excluded, so these sets contain only routed experts. Pooling all active token-layer-expert edges across the full comparison gives
\[
E_A = \{(\ell, t, i) : i \in A_t\}, \qquad
E_B = \{(\ell, t, i) : i \in B_t\}.
\]
The pooled edge sets $E_A$ and $E_B$ are used by the weighted overlap metrics, while the token-level sets $A_t$ and $B_t$ are used by the per-token overlap and divergence metrics below.

\paragraph{Metric definitions.}
Our main metric is weighted Jaccard, defined on the pooled edge sets
\[
\mathrm{weighted\_jaccard}
=
\frac{|E_A \cap E_B|}{|E_A \cup E_B|}.
\]
Its pooled Dice companion is
\[
\mathrm{weighted\_dice}
=
\frac{2|E_A \cap E_B|}{|E_A| + |E_B|}.
\]
We also report token-level overlap averaged uniformly over token-layer pairs
\[
J_t = \frac{|A_t \cap B_t|}{|A_t \cup B_t|}, \qquad
\mathrm{Dice}_t = \frac{2|A_t \cap B_t|}{|A_t| + |B_t|}.
\]
The reported \texttt{jaccard} and \texttt{dice} are the means of $J_t$ and $\mathrm{Dice}_t$ over all token-layer pairs.

For the divergence metrics, we convert each token's binary activation set into a distribution over experts by assigning uniform mass to the active experts
\[
\begin{aligned}
P_t(i) &=
\begin{cases}
1/|A_t|, & i \in A_t \\
0, & \text{otherwise}
\end{cases}
\\
Q_t(i) &=
\begin{cases}
1/|B_t|, & i \in B_t \\
0, & \text{otherwise.}
\end{cases}
\end{aligned}
\]
Using these distributions, we compute
\[
\begin{aligned}
\mathrm{joint\_jsd}(P_t,Q_t)
&=
\frac{1}{2}\mathrm{KL}(P_t \| M_t)
+
\frac{1}{2}\mathrm{KL}(Q_t \| M_t),
\\
M_t &= \frac{1}{2}(P_t + Q_t),
\end{aligned}
\]
and
\[
\mathrm{total\_variation}(P_t,Q_t)
=
\frac{1}{2}\sum_i \lvert P_t(i)-Q_t(i)\rvert.
\]
The reported \texttt{joint\_jsd} and \texttt{total\_variation} are averages over token-layer pairs, and lower values indicate more stable routing.

\paragraph{Empty-routing conventions.}
If both checkpoints activate no routed expert for a token-layer pair, we set token-level Jaccard and Dice to $1$, and joint JSD and total variation to $0$. If only one checkpoint activates any routed expert, we set token-level Jaccard and Dice to $0$, and joint JSD and total variation to $1$. For the pooled metrics, if both pooled edge sets are empty, weighted Jaccard and weighted Dice are both defined as $1$.

\begin{figure*}[t]
    \centering
    \includegraphics[width=\textwidth]{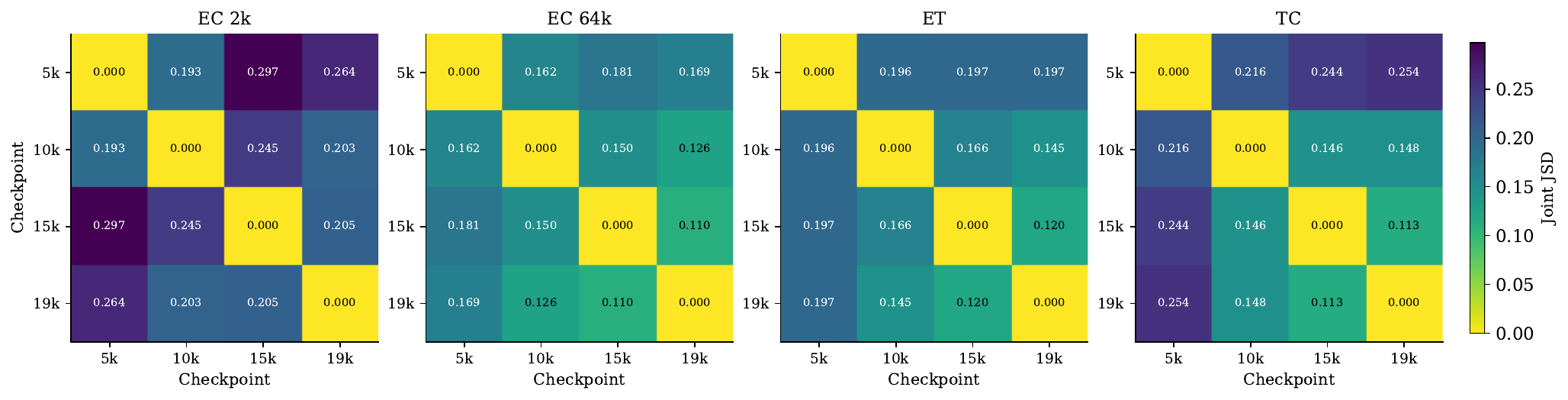}
    \caption{Within-family checkpoint-pair routing consistency on a fixed validation stream, measured by joint JSD. Lower values indicate more stable routing. The same broad story remains. ET separates clearly from EC 2k and stays close to EC 64k.}
    \label{fig:routing_consistency_joint_jsd}
\end{figure*}

\subsection{Activation Dynamics Sweep Across Routing Variants}
\label{app_activation_dynamics_sweep}

Section~\ref{sec:dynamic_computation_allocation} focuses on EC (2k) and ET. This subsection first gives the layerwise continuation for those two main runs using loss binned fanout views, then shows EC 8k in the same overlaid form as the main figure before collecting inverse layerwise diagnostics for the remaining three runs.

\begin{figure*}[t]
    \centering
    \begin{tabular}{@{}cc@{}}
        \includegraphics[width=0.48\textwidth]{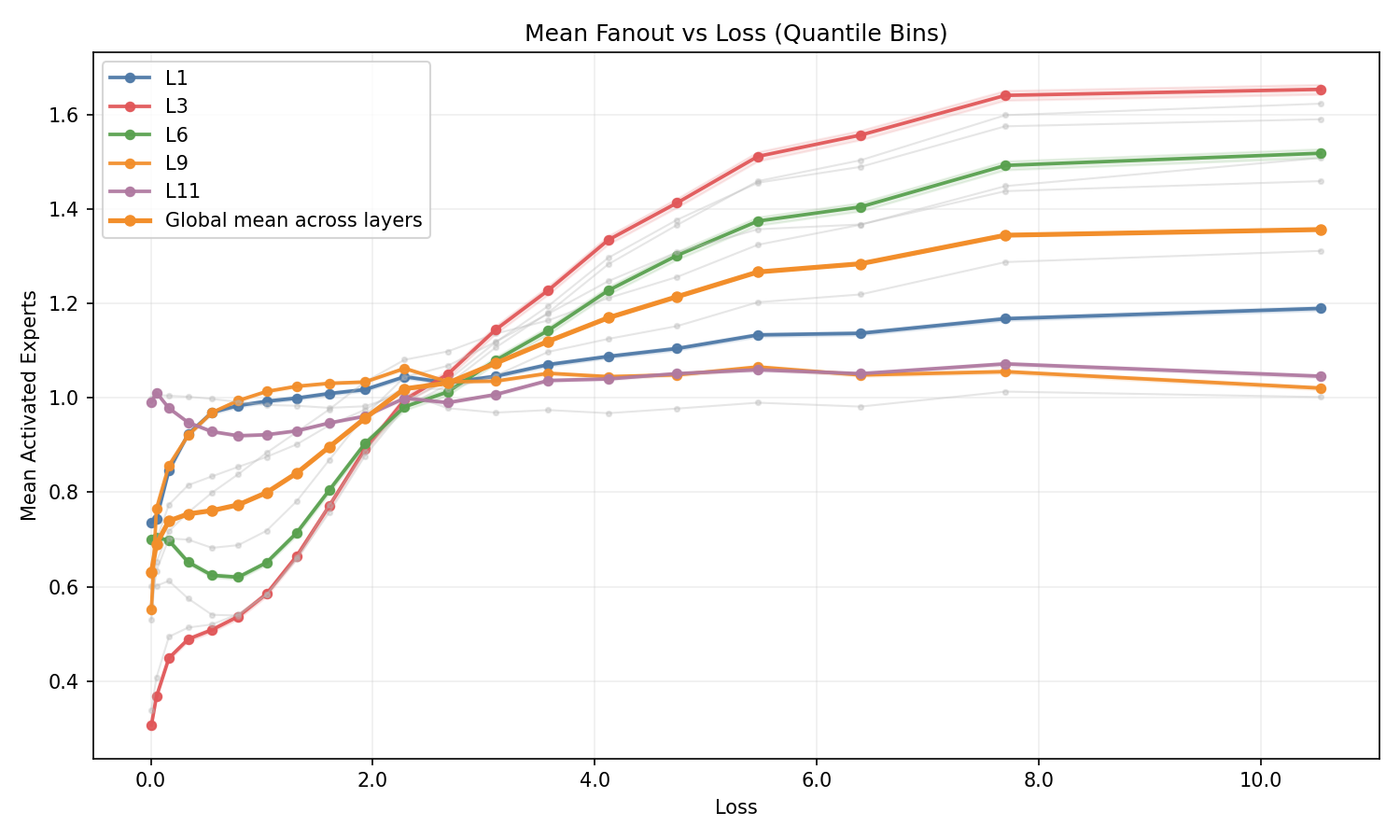} &
        \includegraphics[width=0.48\textwidth]{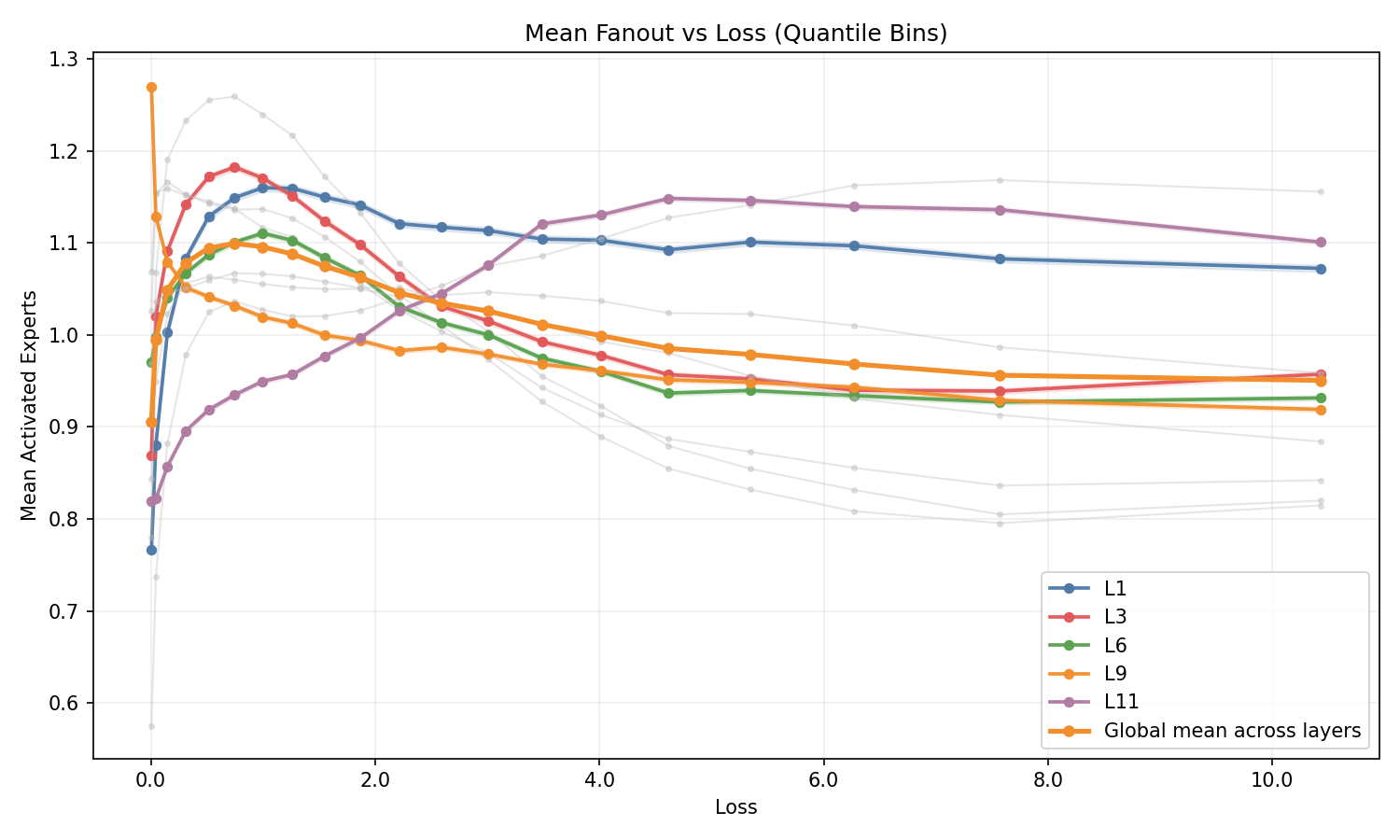} \\
        \scriptsize EC (2k) &
        \scriptsize ET \\
    \end{tabular}
    \caption{Layerwise mean fanout versus loss bin for the two main runs. EC 2k shows a stronger positive dependence in several layers, while ET remains more mixed across depth.}
    \label{fig_activation_layer_main_appendix}
\end{figure*}

\begin{figure*}[t]
    \centering
    \includegraphics[width=0.7\textwidth]{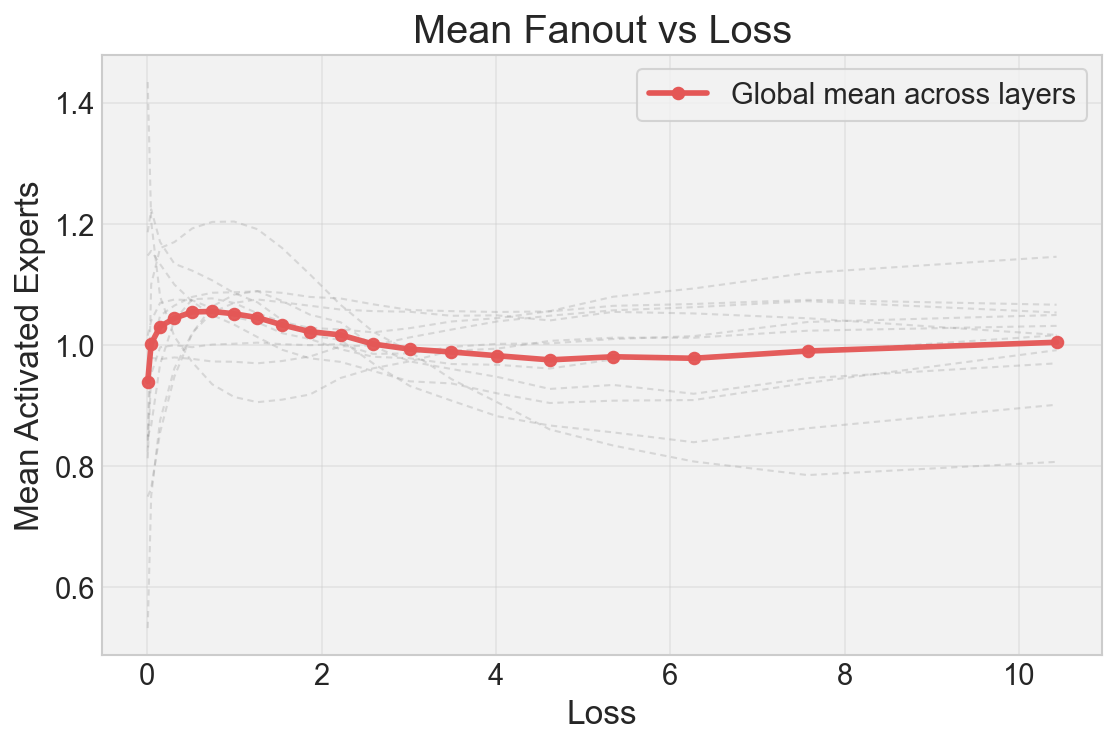}
    \caption{Standalone EC 8k activation dynamics in the same form as the main figure. Faint dashed gray curves show the per-layer means and the solid red curve shows the global mean across layers. Compared with EC 2k, EC 8k is noticeably flatter across the loss range.}
    \label{fig_activation_global_appendix}
\end{figure*}

\begin{figure*}[t]
    \centering
    \begin{tabular}{@{}ccc@{}}
        \includegraphics[width=0.32\textwidth]{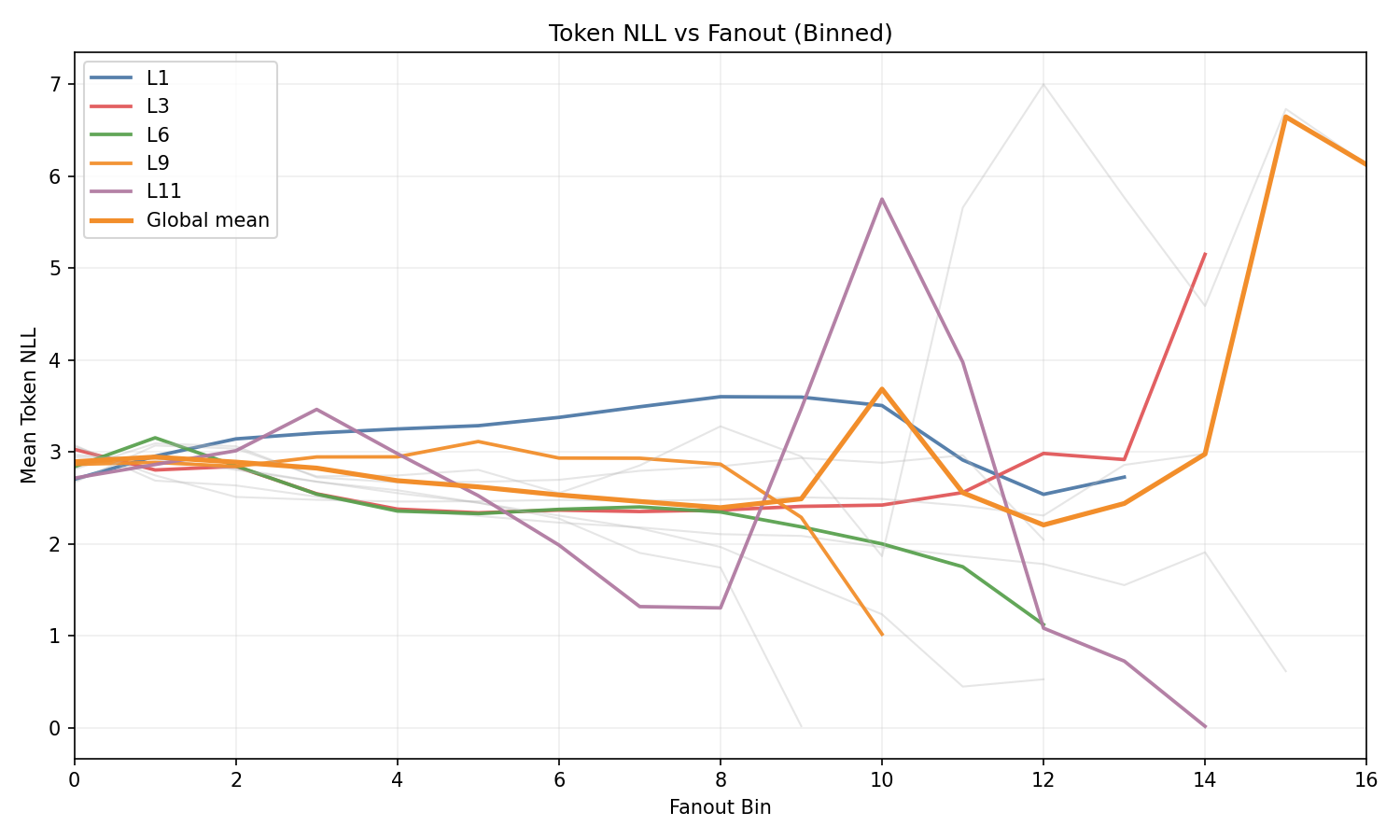} &
        \includegraphics[width=0.32\textwidth]{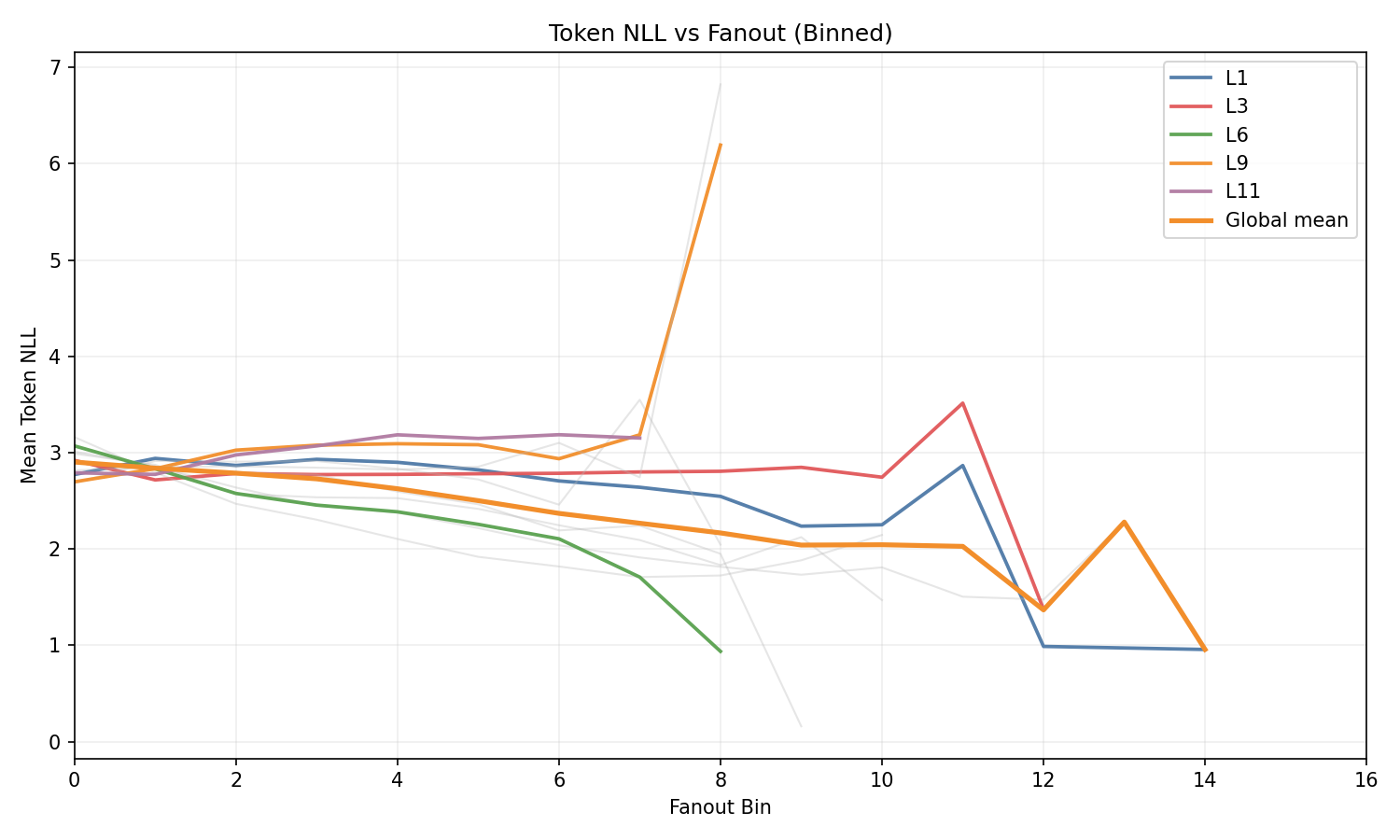} &
        \includegraphics[width=0.32\textwidth]{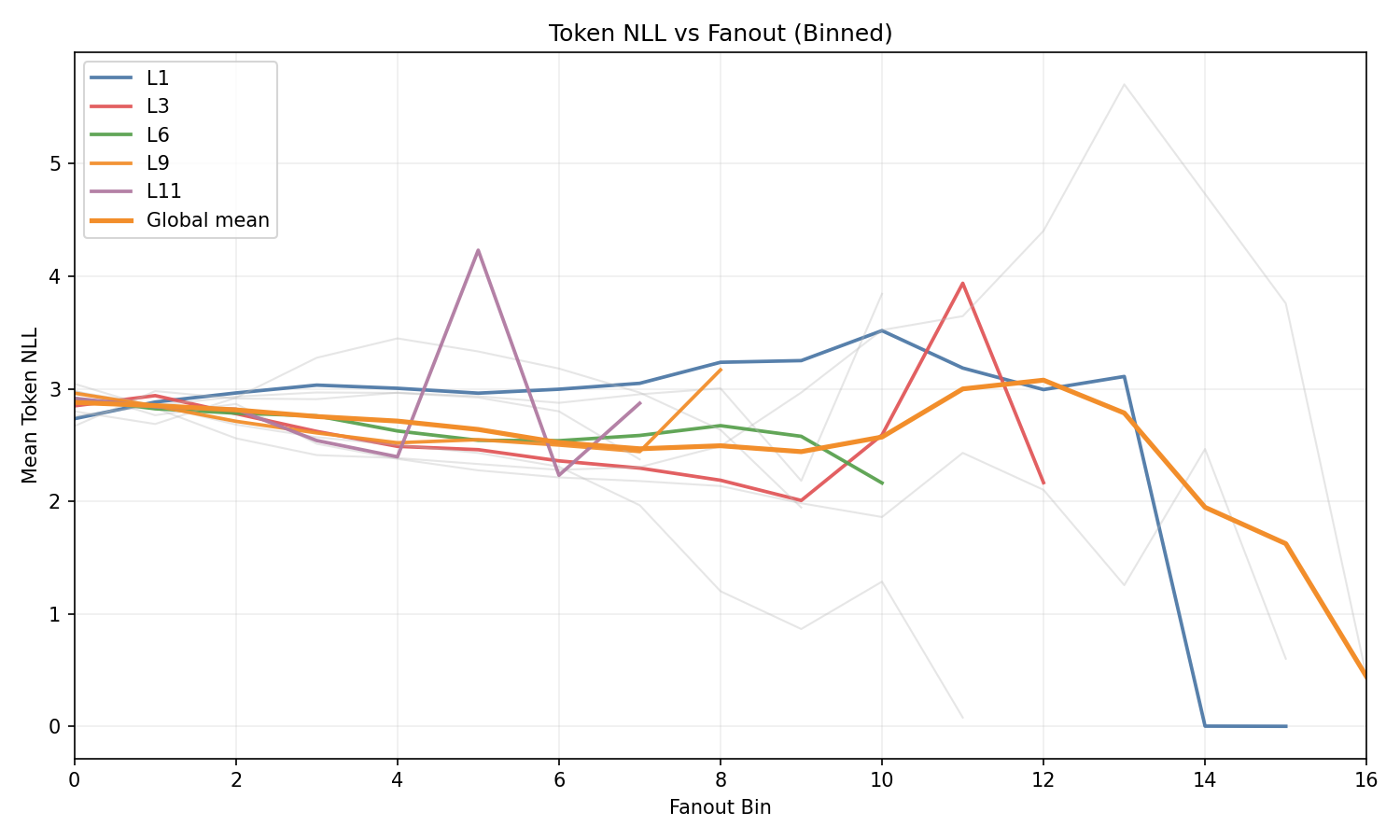} \\
        \scriptsize ET no warmup &
        \scriptsize EC (64k) &
        \scriptsize EC (512k) \\
    \end{tabular}
    \caption{Inverse layerwise activation diagnostics for the remaining three runs. Each panel plots mean loss against fanout by layer, highlighting how the loss--fanout relation remains setup dependent across ET no warmup, EC 64k, and EC 512k.}
    \label{fig_activation_layer_appendix}
\end{figure*}

Across these additional variants, EC 8k is much flatter than EC 2k in the same overlaid inverse view used in the main text. The remaining three inverse layerwise panels show that the loss--fanout relation remains setup dependent, with substantial variation across routing variants and depth.

For completeness, Figure~\ref{fig:routing_logit_hist_appendix} shows representative router logit histograms from the warmup ET run. We include them as a qualitative diagnostic of router behavior.

\begin{figure*}[t]
    \centering
    \begin{tabular}{@{}cc@{}}
        \includegraphics[width=0.48\textwidth]{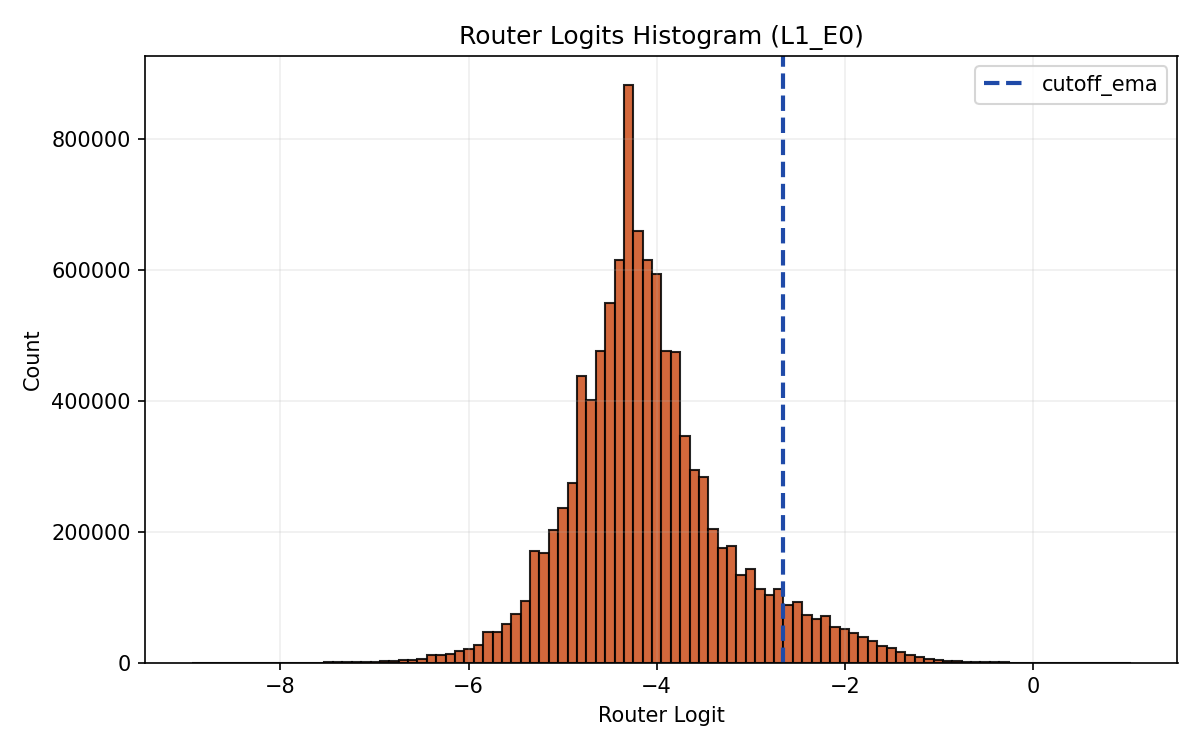} &
        \includegraphics[width=0.48\textwidth]{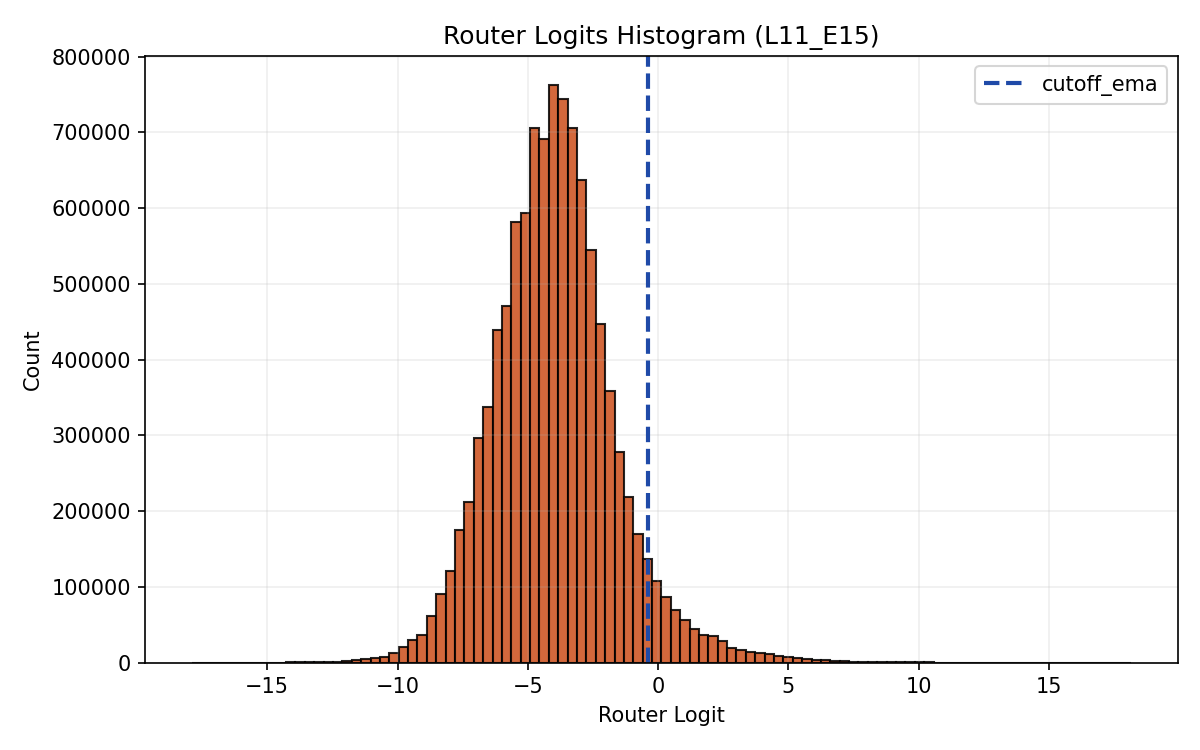} \\
        \scriptsize Layer 1 expert 0 &
        \scriptsize Layer 11 expert 15 \\
    \end{tabular}
    \caption{Router logit histograms from the warmup ET run. The bulk of the distribution is roughly bell shaped, with a heavier right tail. This asymmetric tail is consistent with activated tokens receiving reinforcing gradient signals that can further increase their logits. We view this figure as a qualitative appendix diagnostic rather than a core result.}
    \label{fig:routing_logit_hist_appendix}
\end{figure*}

This appendix also provides extended expert specialization analysis, complementing the summary in Section~\ref{sec:expert_specialization}.

\paragraph{Per-token routing visualizations.}
Figures~\ref{fig:passage_fanout},~\ref{fig:passage_routing}, and~\ref{fig:passage_intensity} show token-level expert routing for additional passages from GSM8K and HumanEval. Across GSM8K passages, content-bearing tokens---particularly numbers (e.g., ``48'', ``72''), mathematical operators (``/'', ``+'', ``=''), and computation markers (``$<<$'')---consistently receive the highest fanout. Function words and punctuation receive minimal activation, indicating that experts preferentially process semantically rich tokens. In HumanEval passages, a similar pattern holds: code-specific tokens (variable names, operators, keywords) receive higher activation than boilerplate text and whitespace.

\begin{figure*}[t]
    \centering
    \begin{minipage}{0.48\textwidth}
        \centering
        \includegraphics[width=\textwidth]{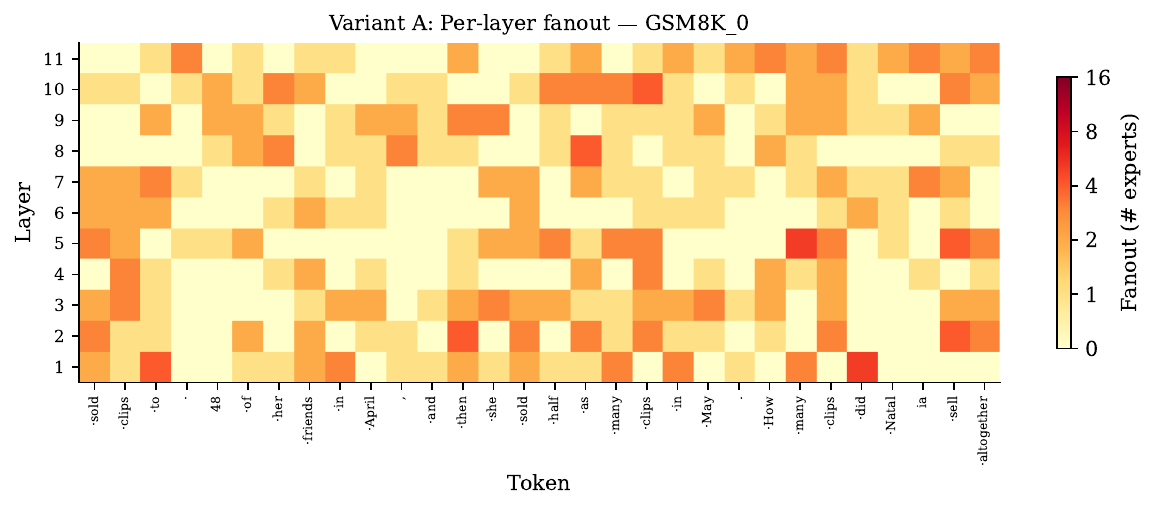}
        \subcaption{GSM8K.}
    \end{minipage}
    \hfill
    \begin{minipage}{0.48\textwidth}
        \centering
        \includegraphics[width=\textwidth]{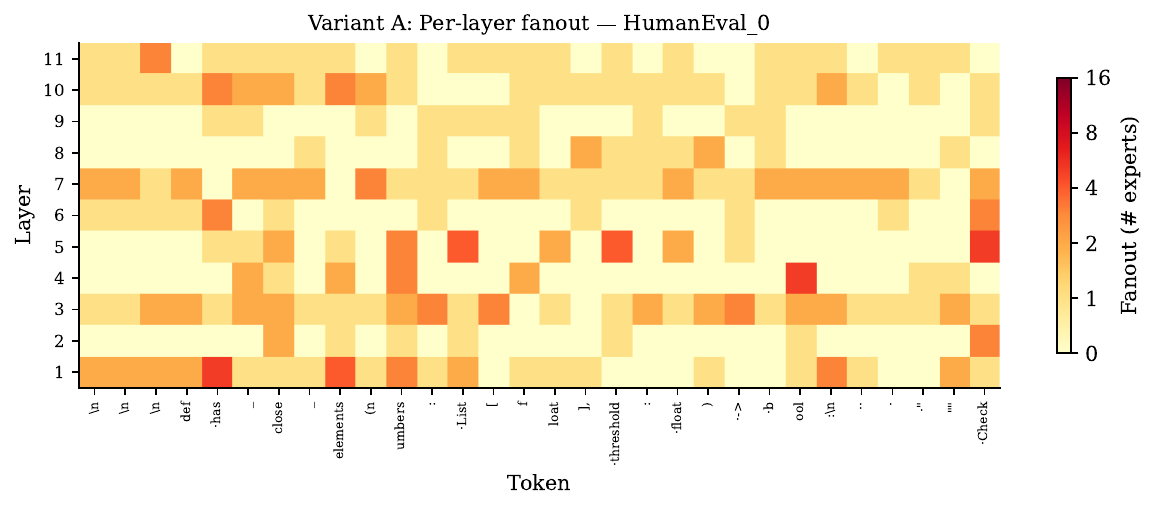}
        \subcaption{HumanEval.}
    \end{minipage}
    \caption{Per-layer expert fanout on GSM8K and HumanEval passages. Each cell shows the number of experts activated for a given token at a given layer. Numerical and code-specific tokens receive substantially higher fanout than function words.}
    \label{fig:passage_fanout}
\end{figure*}

\begin{figure*}[t]
    \centering
    \begin{minipage}{0.48\textwidth}
        \centering
        \includegraphics[width=\textwidth]{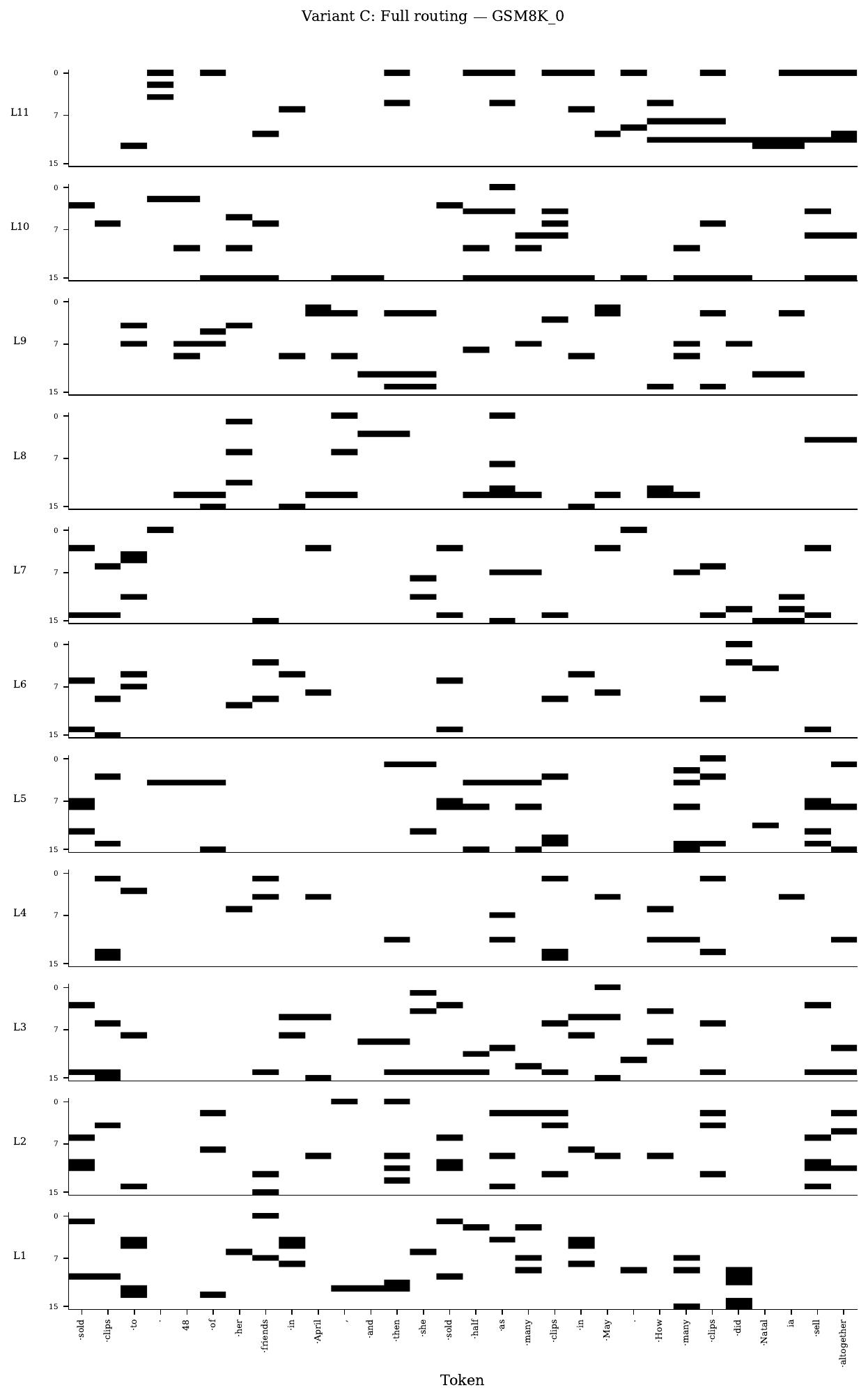}
        \subcaption{GSM8K.}
    \end{minipage}
    \hfill
    \begin{minipage}{0.48\textwidth}
        \centering
        \includegraphics[width=\textwidth]{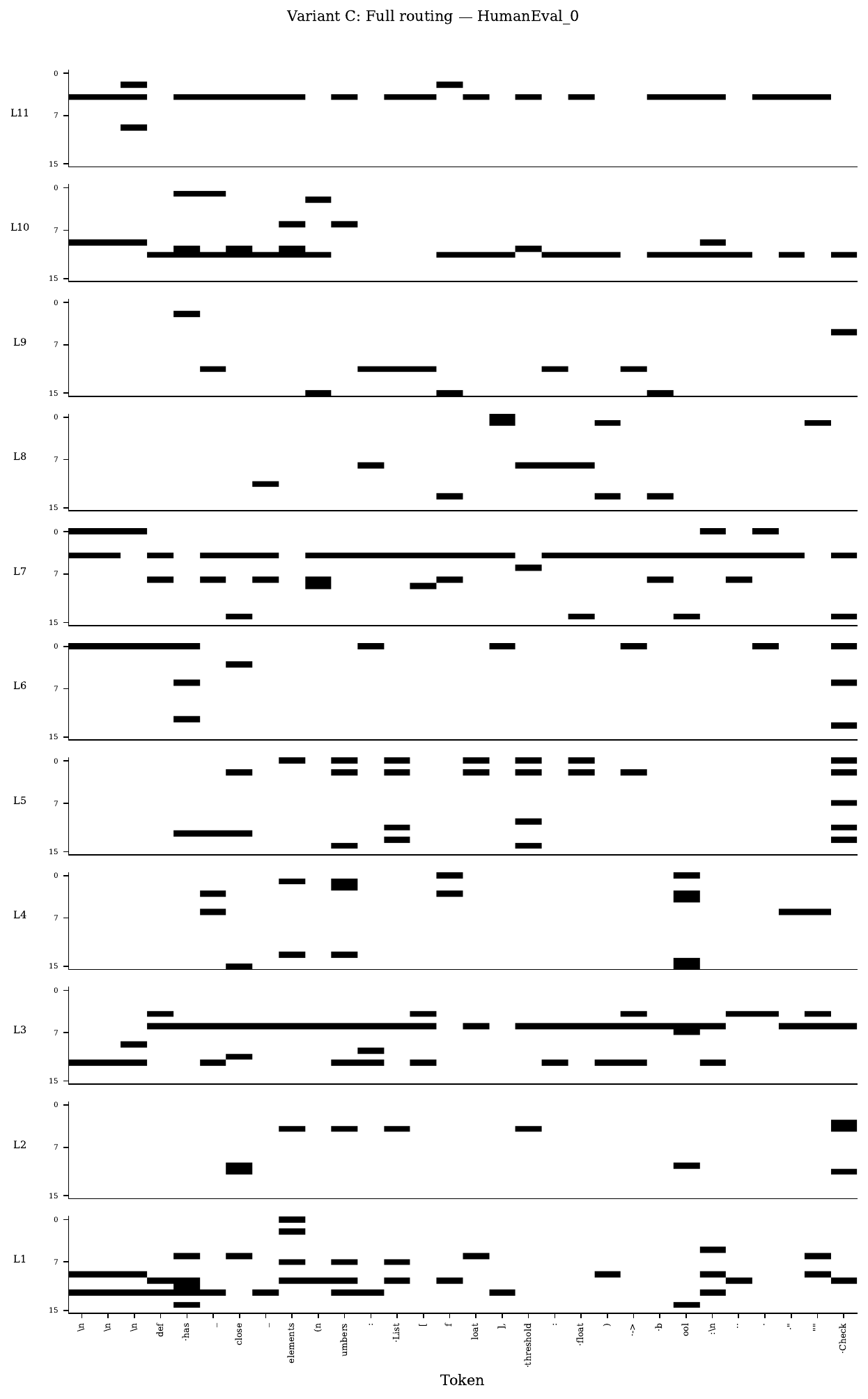}
        \subcaption{HumanEval.}
    \end{minipage}
    \caption{Full expert routing on GSM8K and HumanEval passages. Each panel shows binary expert activation (black = activated) across all layers and experts for every token. Routing patterns reveal domain-specific structure.}
    \label{fig:passage_routing}
\end{figure*}

\begin{figure*}[p]
    \centering
    \includegraphics[width=\textwidth]{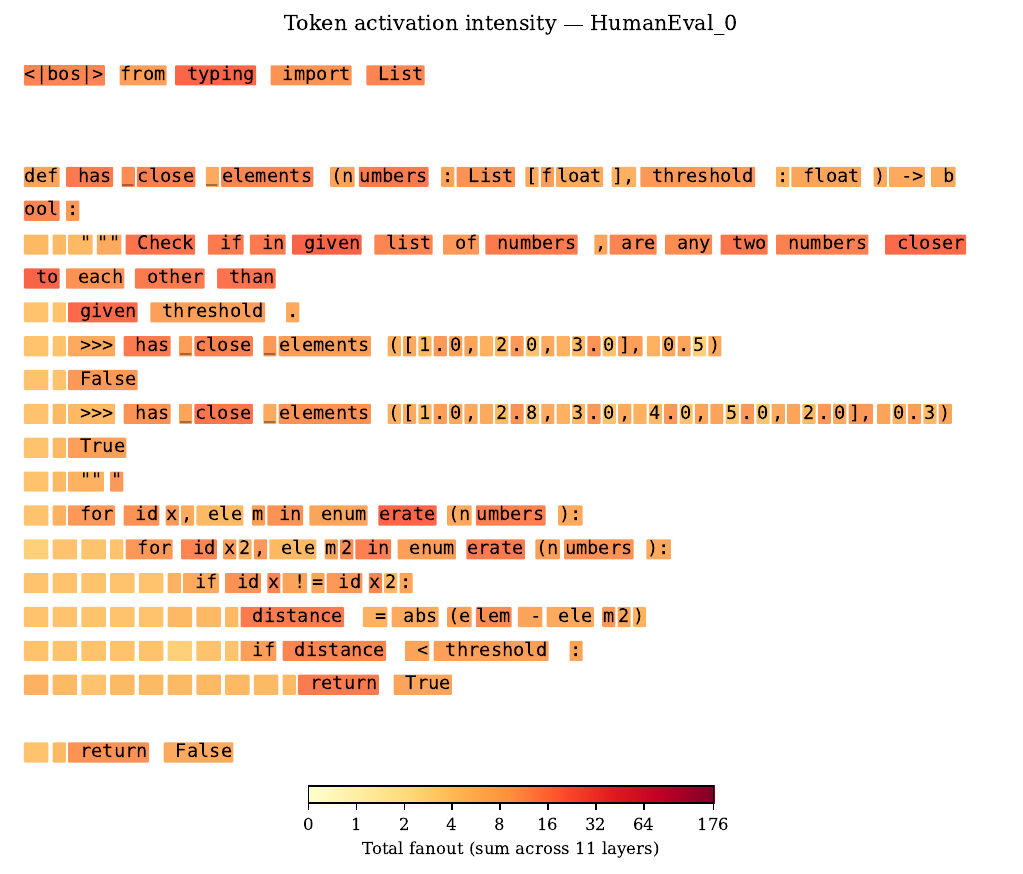}
    \caption{Token activation intensity on a HumanEval passage. Each token is colored by total fanout (sum of experts activated across all layers). Code-specific tokens (variable names, operators, keywords) receive higher activation than boilerplate text.}
    \label{fig:passage_intensity}
\end{figure*}

\paragraph{Expert activation heatmaps.}
Figure~\ref{fig:heatmap_combined} shows expert token ratios across all routing configurations. Each heatmap plots expert ID (columns) versus layer (rows), with color intensity indicating the fraction of domain-specific tokens routed to each expert. The left column shows HumanEval (code) and the right column shows GSM8K (math).

Several patterns emerge across batch sizes. EC with batch size 2k (top row) shows diffuse activation: while some experts exhibit domain preferences (e.g., concentrated dark cells at specific layer-expert pairs), the overall pattern is noisy with activation spread across many experts. As batch size increases to 8k and 64k, specialization sharpens---dark cells become more concentrated and background activation fades, indicating that experts more consistently capture domain-specific tokens when routing decisions are made over larger token pools. EC at 512k shows the most pronounced specialization, with a small number of experts per layer handling the majority of domain tokens.

ET (bottom row) achieves specialization comparable to large-batch EC. The activation patterns closely resemble EC at 512k, with concentrated expert-domain associations across layers. This confirms that ET's population-level threshold mechanism captures the same routing structure as large-batch top-$k$ selection, without requiring batch size coordination at inference.

Comparing across domains, the HumanEval and GSM8K columns reveal that experts develop \emph{different} specialization patterns for code versus math. Certain experts that are heavily activated for code tokens (e.g., dark cells in the HumanEval column) show low activation for math, and vice versa. This cross-domain differentiation is consistent across all routing configurations, suggesting that expert specialization reflects genuine domain-level structure rather than artifacts of a particular routing strategy.

\begin{figure*}[t]
    \centering
    \includegraphics[width=0.7\textwidth]{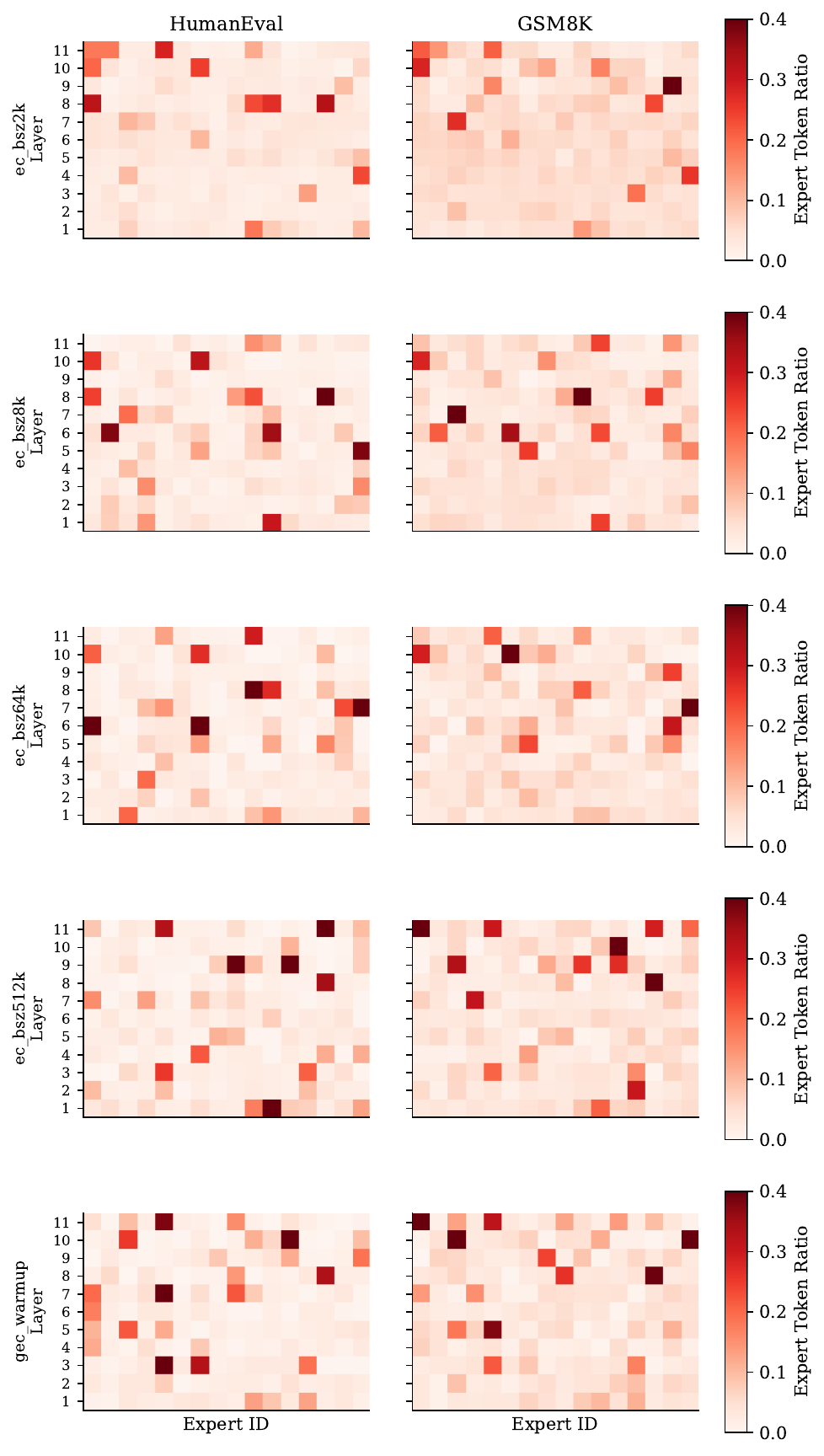}
    \caption{Expert activation heatmaps across routing configurations. Each row corresponds to a routing variant (EC with batch sizes 2k, 8k, 64k, 512k, and ET). Columns show HumanEval (code) and GSM8K (math) domains. Color intensity indicates expert token ratio. Specialization sharpens with larger EC batch sizes, and ET achieves comparable patterns without batch size dependence.}
    \label{fig:heatmap_combined}
\end{figure*}

\end{document}